\documentclass{article}

     \PassOptionsToPackage{numbers, compress}{natbib}

 \usepackage[final]{neurips_2019}

\usepackage[utf8]{inputenc} 

\usepackage[english]{babel}

\usepackage[T1]{fontenc}    
\usepackage{hyperref}       
\usepackage{url}            
\usepackage{booktabs}       
\usepackage{amsfonts}       
\usepackage{nicefrac}       
\usepackage{microtype}      
\usepackage[labelfont=bf]{caption}
\usepackage{booktabs}
\usepackage{siunitx}
\usepackage{comment}
\usepackage{amsmath,amssymb,amsthm}
\usepackage{pifont}
\newcommand{\cmark}{\ding{51}}%
\newcommand{\xmark}{\ding{55}}%
\usepackage{slashbox}
\usepackage{makecell}
\usepackage{multirow}
\usepackage{times}
\usepackage{epsfig}
\usepackage{placeins}
\usepackage{colortbl}
\usepackage{rotating}
\usepackage{array}
\usepackage{color}
\usepackage{xcolor}
\usepackage{verbatim}
\usepackage{units}
\usepackage{textcomp}
\usepackage{bbding}
\usepackage{mwe}
\usepackage{graphicx}
\usepackage{subcaption}

\PassOptionsToPackage{normalem}{ulem}
\usepackage{ulem}
\definecolor{lightgray}{gray}{0.4}
\definecolor{verylightgray}{gray}{0.7}
\definecolor{veryverylightgray}{gray}{0.9}

\definecolor{darkgreen}{rgb}{0, 0.5, 0}

\usepackage{amsmath,amsfonts,bm}

\def\eqref#1{equation~\ref{#1}}

\def\1{\bm{1}}

\def\vs{{\bm{s}}}

\def\vx{{\bm{x}}}

\def\vz{{\bm{z}}}

\DeclareMathAlphabet{\mathsfit}{\encodingdefault}{\sfdefault}{m}{sl}
\SetMathAlphabet{\mathsfit}{bold}{\encodingdefault}{\sfdefault}{bx}{n}

\newcommand{\bvec}[1]{\displaystyle {#1}}

\newtheorem{thm}{Theorem}
\newtheorem{pro}{Proposition}
\newtheorem{lma}{Lemma}
\newcommand{\beginsupplement}{%
	\setcounter{table}{0}
	\renewcommand{\thetable}{S\arabic{table}}%
	\setcounter{figure}{0}
	\renewcommand{\thefigure}{S\arabic{figure}}%
	\setcounter{section}{0}
	\renewcommand{\thepage}{S\arabic{page}} 
	\renewcommand{\thesection}{S\arabic{section}}  
	\setcounter{equation}{0}
	\renewcommand{\theequation}{S\arabic{equation}}
}


\title{Progressive Augmentation of GANs}

	\vspace{0em}
\author{%
	Dan~Zhang\\
	Bosch Center for Artificial Intelligence\\
	\texttt{dan.zhang2@bosch.com} \\
	\And
	Anna~Khoreva \\
	Bosch Center for Artificial Intelligence\\
	\texttt{anna.khoreva@bosch.com} \\
}

\begin{document}

\maketitle

\begin{abstract}
	Training of Generative Adversarial Networks (GANs) is notoriously fragile, 
requiring to maintain a careful balance between the generator and the discriminator in order to perform well. %
To mitigate this issue we introduce a new regularization technique -  \emph{progressive augmentation of GANs (PA-GAN)}. %
The key idea is to gradually increase the task difficulty of the discriminator by progressively augmenting its input or feature space, thus enabling continuous learning of the generator. %
We show that the proposed progressive augmentation preserves the original GAN objective, does not compromise the discriminator's optimality and encourages a healthy competition between the generator and discriminator, leading to the better-performing generator. We experimentally demonstrate the effectiveness of PA-GAN across different architectures and 
on multiple benchmarks for the image synthesis task, on average achieving $\sim 3$ point improvement of the FID score.

\end{abstract}

\section{\label{sec:Introduction}Introduction}
Generative Adversarial Networks (GANs)~\cite{goodfellow2014generative} are a recent development in the field of deep learning, that have attracted a lot of attention in the research community~\cite{Radford2016UnsupervisedRL,SalimansNIPS2016,Arjovsky2017WGAN,karras2018progressive}. The GAN framework can be formulated as a competing game between the generator and the discriminator. Since both the generator and the discriminator are typically parameterized as deep convolutional neural networks with millions of parameters, optimization is notoriously difficult in practice~\cite{Arjovsky2017WGAN,gulrajani_NIPS2017,miyato2018spectral}.

The difficulty lies in maintaining a healthy competition between the generator and discriminator. A commonly occurring problem arises when the discriminator overshoots, leading to escalated gradients and oscillatory GAN behaviour~\cite{MeschederICML2018, Brock2019}. %
Moreover, the supports of the data and model distributions typically lie on low dimensional manifolds and are often disjoint~\cite{Arjovsky2017TowardsPM}. Consequently, there exists a nearly trivial discriminator that can perfectly distinguish real data samples from synthetic ones. Once such a discriminator is produced, its loss quickly converges to zero and the gradients used for updating parameters of the generator become useless. For improving the training stability of GANs regularization techniques~\cite{Roth_NIPS2017,gulrajani_NIPS2017} can be used to constrain the learning of the discriminator. But as shown in~\cite{Brock2019,Kurach2018GANlandscape} they also impair the generator and lead to the performance degradation.

In this work we introduce a new regularization technique to alleviate this problem  - \emph{progressive augmentation of GANs (PA-GAN)} - that helps to control the behaviour of the discriminator and thus improve the overall training.\footnote{\url{https://github.com/boschresearch/PA-GAN}} The key idea is to progressively augment the input of the discriminator network or its intermediate feature layers with auxiliary random bits in order to gradually increase the discrimination task difficulty (see Fig.~\ref{fig:Data_augmentation}). In doing so, the discriminator can be prevented from becoming over-confident, enabling continuous learning of the generator. %
As opposed to standard augmentation techniques (e.g. rotation, cropping, resizing), the proposed progressive augmentation does not directly modify the data samples or their features, but rather structurally appends to them. Moreover, it can also alter the input class. For instance, in the single-level augmentation the data sample or its features $\bvec{\vx}$ are combined with a random bit $s$ and both are provided to the discriminator. The class of the augmented sample $(\bvec{\vx},s)$ is then set based on the combination $\bvec{\vx}$ with $s$, resulting in real and synthetic samples contained in both classes, see Fig.~\ref{fig:Data_augmentation}-(a). %
This presents a more challenging task for the discriminator, as it needs to tell the real and synthetic samples apart plus additionally learn how to separate $(\bvec{\vx},s)$ back into $\bvec{\vx}$ and $s$ and understand the association rule. We can further increase the task difficulty of the discriminator by progressively augmenting its input or feature space, gradually increasing the number of random bits during the course of training as depicted in Fig.~\ref{fig:Data_augmentation}-(b).

We prove that PA-GAN preserves the original GAN objective and, in contrast to prior work~\cite{Arjovsky2017TowardsPM,Sonderby2016AmortisedMI,SalimansNIPS2016}, does not bias the optimality of the discriminator (see Sec.~\ref{sec:PAGAN}). Aiming at minimum changes we further propose an integration of PA-GAN into existing GAN architectures (see Sec.~\ref{subsec:implement}) and experimentally showcase its benefits (see Sec.~\ref{sec:PA-diff}). %
Structurally augmenting the input or its features and mapping them to higher dimensions not only challenges the discrimination task, but, in addition, with each realization of the random bits alters the loss function landscape, potentially providing a different path for the generator to approach the data distribution.

Our technique is orthogonal to existing work, it can be successfully employed with other regularization strategies~\cite{Roth_NIPS2017,gulrajani_NIPS2017,SalimansNIPS2016,JMLR:v15:srivastava14a,ChenSS2019} and different network architectures \cite{miyato2018spectral,Zhang_SAGAN18}, which we demonstrate in Sec.~\ref{subsec:Exp-Regulariz}. We experimentally show the effectiveness of PA-GAN for unsupervised image generation tasks on multiple benchmarks (Fashion-MNIST~\cite{xiao2017}, CIFAR10 \cite{Cifar10_Krizhevsky09learningmultiple}, CELEBA-HQ~\cite{karras2018progressive}, and Tiny-ImageNet~\cite{imagenet_cvpr09}), on average improving the FID score around $3$ points. For PA combination with SS-GAN~\cite{ChenSS2019} we achieve the best FID of $14.7$ for the unsupervised setting on CIFAR10, which is on par with the results achieved by large scale BigGAN training~\cite{Brock2019} using label supervision.

\begin{figure}[t!]%
	\vspace{-1em}
	\centering
	\begin{subfigure}{0.4\textwidth}
			\centering
\includegraphics[width=\linewidth]{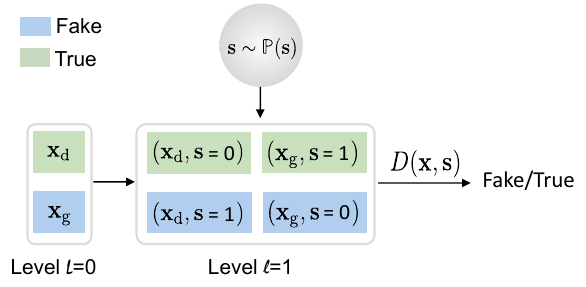} \caption{Discriminator task}
	\end{subfigure}
\begin{subfigure}{0.45\textwidth}
		\centering
\includegraphics[width=\linewidth]{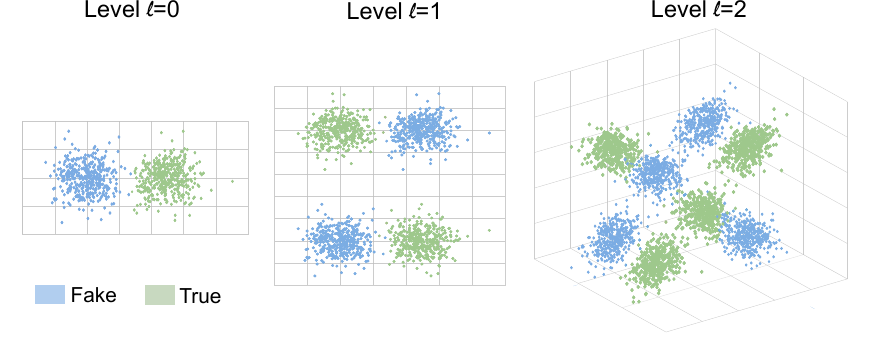} \caption{Input space augmentation}
\end{subfigure}
	\caption{\label{fig:Data_augmentation}
		Visualization of progressive augmentation.
		At level $l=0$ (no augmentation) the discriminator $D$ aims at classifying the samples $\bvec{\vx}_{\mathrm{d}}$ and  $\bvec{\vx}_{\mathrm{g}}$, respectively drawn from the data $\mathbb{P}_{\mathrm{d}}$ and generative model $\mathbb{P}_{\mathrm{g}}$ distributions, into true (green) and fake (blue). At single-level augmentation ($l=1$) the class of the augmented sample is set based on the combination $\bvec{\vx}_{\mathrm{d}}$ and $\bvec{\vx}_{\mathrm{g}}$ with $s$, resulting in real and synthetic samples contained in both classes and leading to a harder task for $D$.
		With each extra augmentation level ($l \rightarrow l+1$) the decision boundary between two classes becomes more complex and the discrimination task difficulty gradually increases. This prevents the discriminator from easily solving the task and thus leads to meaningful gradients for the generator updates.}
		\vspace{-1em}
\end{figure}

\section{\label{sec:Related}Related Work}	

Many recent works have focused on improving the stability of GAN training and the overall visual quality of generated samples \cite{Roth_NIPS2017,miyato2018spectral,Zhang_SAGAN18, Brock2019}.
The unstable behaviour of GANs is partly attributed to a dimensional mismatch or non-overlapping support between the real data and the generative model distributions \cite{Arjovsky2017TowardsPM}, 
resulting in an almost trivial task for the discriminator. Once the performance of the discriminator is maxed out, it provides a non-informative signal to train the generator. To avoid vanishing gradients, the original GAN paper \cite{goodfellow2014generative} proposed to modify the min-max based GAN objective to a non-saturating loss.
However, even with such a re-formulation the generator updates tend to get worse over the course of training and optimization becomes massively unstable \cite{Arjovsky2017TowardsPM}.

Prior approaches tried to mitigate this issue by using heuristics to weaken the discriminator, e.g. decreasing its learning rate, adding label noise or directly modifying the data samples.
\cite{SalimansNIPS2016} proposed a one-sided label smoothing to smoothen the classification boundary of the discriminator, thereby preventing it from being overly confident, but at the same time biasing its optimality. %
\cite{Arjovsky2017TowardsPM,Sonderby2016AmortisedMI} tried to ensure a joint support of the data and model distributions to make the job of the discriminator harder by adding Gaussian noise to both generated and real samples. %
However, adding high-dimensional noise introduces significant variance in the parameter estimation, slowing down the training and requiring multiple samples for counteraction \cite{Roth_NIPS2017}. %
Similarly, \cite{SajParMehSch18} proposed to blur the input samples and gradually remove the blurring effect during the course of training. These techniques perform direct modifications on the data samples.

Alternatively, several works focused on regularizing the discriminator.
\cite{gulrajani_NIPS2017} proposed to add a soft penalty on the gradient norm which ensures a 1-Lipschitz discriminator. %
Similarly, \cite{Roth_NIPS2017} added a zero-centered penalty on the weighted gradient-norm of the discriminator, showing its equivalence to adding input noise.
On the downside, regularizing the discriminator with the gradient penalty depends on the model distribution, which changes during training, and results in increased runtime due to additional gradient norm computation \cite{Kurach2018GANlandscape}. Most recently, \cite{Brock2019} also experimentally showed that the gradient penalty may lead to the performance degradation, which corresponds to our observations as well (see Sec. \ref{subsec:Exp-Regulariz}) 
In addition to the gradient penalty, \cite{Brock2019} also exploited the dropout regularization \cite{JMLR:v15:srivastava14a} on the final layer of the discriminator and reported its similar stabilizing effect.
\cite{miyato2018spectral} proposed another way to stabilize the discriminator by normalizing its weights and limiting the spectral norm of each layer to constrain the Lipschitz constant. This normalization technique does not require intensive tuning of hyper-parameters and is computationally light.
Moreover, \cite{Zhang_SAGAN18} showed that spectral normalization is also beneficial for the generator, preventing the escalation of parameter magnitudes and avoiding unusual gradients. 

Several methods have proposed to modify the GAN training methodology in order to further improve stability, e.g. by considering multiple discriminators \cite{Durugkar2016GenerativeMN}, growing both the generator and discriminator networks progressively \cite{karras2018progressive} or exploiting different learning rates for the discriminator and generator \cite{heuselttur2017}. %
Another line of work resorts to objective function reformulation, e.g. by using the Pearson $\chi^2$ divergence~\cite{MaoLXLW16}, the Wasserstein distance \cite{Arjovsky2017WGAN}, or f-divergence \cite{Nowozin2016fGANTG}.

In this work we introduce a novel and orthogonal way of regularizing GANs by progressively increasing the discriminator task difficulty. %
In contrast to other techniques, our method does not bias the optimality of the discriminator or alter the training samples.
Furthermore, the proposed augmentation is complementary to prior work. It can be employed with different GAN architectures and combined with other regularization techniques (see Sec. \ref{sec:Experiments}).

\section{Progressive Augmentation of GANs} \label{sec:PAGAN-all}
\subsection{Theoretical Framework of PA-GAN} \label{sec:PAGAN}
The core idea behind the GAN training~\cite{goodfellow2014generative} is to set up a competing game between two players, commonly termed discriminator and generator. The discriminator aims at distinguishing the samples $\bvec{\vx}\in\mathcal{X}$ respectively drawn from the data distribution $\mathbb{P}_{\mathrm{d}}$ and generative model distribution $\mathbb{P}_{\mathrm{g}}$, i.e. performing binary classification $D:\mathcal{X}\mapsto [0,1]$. 
\footnote{$D(x)$ aims to learn the probability of $\bvec{\vx}$ being true or fake, however, it can also be regarded as the sigmoid response of classification with cross entropy loss.} 
The aim of the generator, on the other hand, is to make synthetic samples into data samples, challenging the discriminator. In this work, $\mathcal{X}$ represents a compact metric space such as the image space $[-1,1]^N$ of dimension $N$. Both $\mathbb{P}_{\mathrm{d}}$ and $\mathbb{P}_{\mathrm{g}}$ are defined on $\mathcal{X}$. The model distribution $\mathbb{P}_{\mathrm{g}}$ is induced by a function $G$ that maps a random vector $\bvec{\vz}\sim\mathbb{P}_{\mathrm{z}}$ to a synthetic data sample, i.e. $\bvec{\vx}_{\mathrm{g}}=G(\bvec{\vz})\in\mathcal{X}$. Mathematically, the two-player game is formulated as
\begin{align}
\min_G \max_D 
\mathbb{E}_{\mathbb{P}_{\mathrm{d}}}\left\{\log \left[D(\bvec{\vx})\right]\right\}   + \mathbb{E}_{\mathbb{P}_{\mathrm{g}}}\left\{\log\left[1-D(\bvec{\vx})\right]\right\}.\label{Mmin-max}
\end{align}
As being proved by~\cite{goodfellow2014generative}, the inner maximum equals the Jensen-Shannon (JS) divergence between $\mathbb{P}_{\mathrm{d}}$ and $\mathbb{P}_{\mathrm{g}}$, i.e., $D_{\mathrm{JS}}\left(\mathbb{P}_{\mathrm{d}}\Vert\mathbb{P}_{\mathrm{g}} \right)$. Therefore, the GAN training attempts to minimize the JS divergence between the model and data distributions.

\begin{lma}\label{Mlma1}
Let $s\in\{0,1\}$ denote a random bit with uniform distribution $\mathbb{P}_{\mathrm{s}}(s)=\frac{\delta[s]+\delta[s-1]}{2}$, where $\delta[s]$ is the Kronecker delta. Associating $s$ with $\bvec{\vx}$, two joint distributions of $(\bvec{\vx}, s)$ are constructed as 
\begin{align}
\mathbb{P}_{\mathrm{x},\mathrm{s}}(\bvec{\vx},s) \stackrel{\Delta}{=} \frac{\mathbb{P}_{\mathrm{d}}(\bvec{\vx})\delta[s]+\mathbb{P}_{\mathrm{g}}(\bvec{\vx})\delta[s-1]}{2} ,\quad \mathbb{Q}_{\mathrm{x},\mathrm{s}}(\bvec{\vx},s) \stackrel{\Delta}{=} \frac{\mathbb{P}_{\mathrm{g}}(\bvec{\vx})\delta[s]+\mathbb{P}_{\mathrm{d}}(\bvec{\vx})\delta[s-1]}{2}.\label{MPQ0}
\end{align}
Their JS divergence is equal to 
\begin{align}
D_{\mathrm{JS}}\left(\mathbb{P}_{\mathrm{x},\mathrm{s}}\Vert\mathbb{Q}_{\mathrm{x},\mathrm{s}} \right)=D_{\mathrm{JS}}\left(\mathbb{P}_{\mathrm{d}}\Vert\mathbb{P}_{\mathrm{g}} \right).\label{eqjoint}
\end{align}
Taking (\ref{MPQ0}) as the starting point and with $\bvec{\vs}_{l}$ being a sequence of i.i.d. random bits of length $l$, the recursion of constructing the paired joint distributions of $(\bvec{\vx},\bvec{\vs}_{l})$ 
\begin{align}
\begin{array}{ll}
\mathbb{P}_{\mathrm{x},\mathbf{s}_{l}}(\bvec{\vx},\bvec{\vs}_{l})  \stackrel{\Delta}{=} {\mathbb{P}_{\mathrm{x},\mathbf{s}_{l-1}}(\bvec{\vx},\bvec{\vs}_{l-1})\delta[s_l]/2+\mathbb{Q}_{\mathrm{x},\mathbf{s}_{l-1}}(\bvec{\vx},\bvec{\vs}_{l-1})\delta[s_l-1]/2}\\
\mathbb{Q}_{\mathrm{x},\mathbf{s}_{l}}(\bvec{\vx},\bvec{\vs}_{l}) \stackrel{\Delta}{=} {\mathbb{Q}_{\mathrm{x},\mathbf{s}_{l-1}}(\bvec{\vx},\bvec{\vs}_{l-1})\delta[s_l]/2+\mathbb{P}_{\mathrm{x},\mathbf{s}_{l-1}}(\bvec{\vx},\bvec{\vs}_{l-1})\delta[s_l-1]/2}
\end{array}\label{MPQL}
\end{align}
results into a series of JS divergence equalities for $l=1,2,\dots,L$, i.e.,
\begin{align}
 D_{\mathrm{JS}}\left(\mathbb{P}_{\mathrm{d}}\Vert\mathbb{P}_{\mathrm{g}} \right)
=D_{\mathrm{JS}}\left(\mathbb{P}_{\mathrm{x},\mathbf{s}_1}\Vert \mathbb{Q}_{\mathrm{x},\mathbf{s}_1} \right)=\cdots=
D_{\mathrm{JS}}\left(\mathbb{P}_{\mathrm{x},\mathbf{s}_L}\Vert \mathbb{Q}_{\mathrm{x},\mathbf{s}_L} \right)
\label{eqjoint2}.
\end{align}
\end{lma}

\begin{thm}\label{Mthm}
The min-max optimization problem of GANs~\cite{goodfellow2014generative} as given in (\ref{Mmin-max}) is equivalent to
\begin{align}
\min_G \max_{D} 
\mathbb{E}_{\mathbb{P}_{\mathrm{x},\mathbf{s}_{l}}}\left\{\log \left[D(\bvec{\vx},\bvec{\vs}_l)\right]\right\}   + \mathbb{E}_{\mathbb{Q}_{\mathrm{x},\mathbf{s}_{l}}}\left\{\log\left[1-D(\bvec{\vx},\bvec{\vs}_l)\right]\right\} \quad\forall l\in\{1,2,\dots,L\},\label{Mmin-max2}
\end{align}
where the two joint distributions, i.e., $\mathbb{P}_{\mathrm{x},\mathbf{s}_{l}}$ and $\mathbb{Q}_{\mathrm{x},\mathbf{s}_{l}}$, are defined in (\ref{MPQL}) and the function $D$ maps $(\bvec{\vx},\bvec{\vs}_l)\in\mathcal{X}\times \{0,1\}^l$ onto $[0,1]$. For a fixed $G$, the optimal $D$ is
\begin{align}
D^*(\bvec{\vx},\bvec{\vs}_l)=\frac{\mathbb{P}_{\mathrm{x},\mathbf{s}_{l}}(\bvec{\vx},\bvec{\vs}_l)}{\mathbb{P}_{\mathrm{x},\mathbf{s}_{l}}(\bvec{\vx},\bvec{\vs}_l)+\mathbb{Q}_{\mathrm{x},\mathbf{s}_{l}}(\bvec{\vx},\bvec{\vs}_l)}=\frac{\mathbb{P}_{\mathrm{d}}(\bvec{\vx})}{\mathbb{P}_{\mathrm{d}}(\bvec{\vx})+\mathbb{Q}_{\mathrm{d}}(\bvec{\vx})},
\end{align}
whereas the attained inner maximum equals $D_{\mathrm{JS}}\left(\mathbb{P}_{\mathrm{x},\mathbf{s}_l}\Vert \mathbb{Q}_{\mathrm{x},\mathbf{s}_l}\right)= D_{\mathrm{JS}}\left(\mathbb{P}_{\mathrm{d}}\Vert\mathbb{P}_{\mathrm{g}} \right)$ for $l=1,2,\dots,L$.
\end{thm}

According to Theorem~\ref{Mthm}, solving (\ref{Mmin-max}) is interchangeable with solving (\ref{Mmin-max2}). In fact, the former can be regarded as a corner case of the latter by taking $l=0$ as the absence of the auxiliary bit vector $\bvec{\vs}$. As the length $l$ of $\bvec{\vs}$ increases, the input dimension of the discriminator grows accordingly. Furthermore, two classes to be classified consist of both the data and synthetic samples as illustrated in Fig.~\ref{fig:Data_augmentation}-(a).  Note that, the mixture strategy of the distributions of two independent random variables in Lemma~\ref{Mlma1} can be extended for any generic random variables (see Sec.~\ref{Ssub:lemma1_gen} in the supp. material). 

When solving (\ref{Mmin-max}), $G$ and $D$ are parameterized as deep neural networks and SGD (or its variants) is typically used for the optimization, updating their weights in an alternating or simultaneous manner, with no guarantees on global convergence. 
Theorem~\ref{Mthm} provides a series of JS divergence estimation proxies by means of the auxiliary bit vector $\bvec{\vs}$ that in practice can be exploited as a regularizer to improve the GAN training (see Sec. \ref{sec_toy_example} for empirical evaluation). First, the number of possible combinations of the data samples with $\bvec{\vs}_l$ grows exponentially with $l$, thus helping to prevent the discriminator from overfitting to the training set. %
Second, the task of the discriminator gradually becomes harder with the length $l$. The input dimensionality of $D$ becomes larger and as the label of $(\bvec{\vx},\bvec{\vs}_{l-1})$ is altered based on the new random bit $s_l$ the decision boundary becomes more complicated (Fig.~\ref{fig:Data_augmentation}-b). Given that, progressively increasing $l$ can be exploited during training to balance the game between the discriminator and generator whenever the former becomes too strong. %
Third, when the GAN training performance saturates at the current augmentation level, adding one random bit changes the landscape of the loss function and may further boost the learning.

\subsection{Implementation of PA-GAN} \label{subsec:implement}

\begin{figure}
	\vspace{-1em}
	\begin{center}
		\includegraphics[width=.7\textwidth]{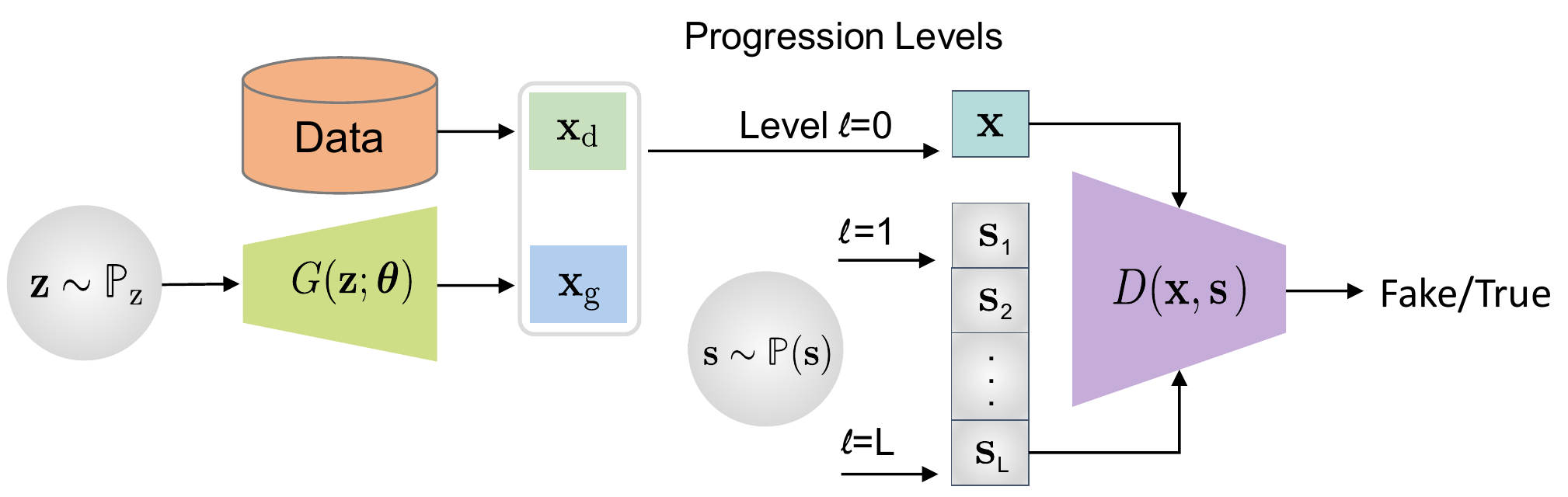}
	\end{center}
	\vspace{-0.5em}
	\caption{\label{fig:approach} PA-GAN overview. With each level of progressive augmentation $l$ the dimensionality of $\bvec{\vs}$ is enlarged from 1 to $L$, $\bvec{\vs}=\{s_1,s_2,\dots,s_L\}$. The task difficulty of the discriminator gradually increases as the length of $\bvec{\vs}$ grows.}
	\vspace{-1em}
\end{figure}
The min-max problem in (\ref{Mmin-max2}) shares the same structure as the original one in (\ref{Mmin-max}), thus we can exploit the standard GAN training for PA-GAN, see Fig.~\ref{fig:approach}. The necessary change only concerns the discriminator. It involves 1) using checksum principle as a new classification criterion, 2) incorporating $\bvec{\vs}$ in addition to $\bvec{\vx}$ as the network input and 3) enabling the progression of $\bvec{\vs}$ during training.

\textbf{Checksum principle.}
The conventional GAN discriminator assigns TRUE (0) / FAKE (1) class label based on $\bvec{\vx}$ being either data or synthetic samples. In contrast, the discriminator $D$ in (\ref{Mmin-max2}) requires $\bvec{\vs}_l$ along with $\bvec{\vx}$ to make the decision about the class label. Starting from $l=1$, the two class distributions in (\ref{MPQ0}) imply the label-$0$ for $(\bvec{\vx}_{\mathrm{d}},s=0)$, $(\bvec{\vx}_{\mathrm{g}},s=1)$ and label-$1$ for $(\bvec{\vx}_{\mathrm{d}},s=1)$, $(\bvec{\vx}_{\mathrm{g}},s=0)$. The real samples are no longer always in the TRUE class, and the synthetic samples are no longer always in the FAKE class, see Fig.~\ref{fig:Data_augmentation}-(a). 
To detect the correct class we can use a simple checksum principle.
Namely, let the data and synthetic samples respectively encode bit $0$ and $1$ followed by associating the checksum $0(1)$ of the pair $(\bvec{\vx},s)$ with TRUE(FAKE). \footnote{By checksum we mean the XOR operation over a bit sequence.} %
For more than one bit, $\mathbb{P}_{\mathrm{x},\mathbf{s}_{l}}$ and $\mathbb{Q}_{\mathrm{x},\mathbf{s}_{l}}$ are recursively constructed according to (\ref{MPQL}). Based on the checksum principle for the single bit case, we can recursively show its consistency for any bit sequence length $\bvec{\vs}_l$, $l>1$. This is a desirable property for progression. %
With the identified checksum principle, we further discuss a way to integrate a sequence of random bits $\bvec{\vs}_l$ into the discriminator network in a progressive manner.

\textbf{Progressive augmentation.}\label{subsec:network-implement}
With the aim of maximally reusing existing GAN architectures we propose two augmentation options. The first one is \emph{input space augmentation}, where $\bvec{\vs}$ is directly concatenated with the sample $\bvec{\vx}$ and both are fed as input to the discriminator network. %
The second option is \emph{feature space augmentation},  where $\bvec{\vs}$ is concatenated with the learned feature representations of $\bvec{\vx}$ attained at intermediate hidden layers. For both cases, the way to concatenate $\bvec{\vs}$ with $\bvec{\vx}$ or its feature maps is identical.
Each entry $s_l$ creates one augmentation channel, which is replicated to match the spatial dimension of $\bvec{\vx}$ or its feature maps. %
Depending on the augmentation space, either the input layer or the hidden layer that further processes the feature maps will additionally take care of the augmentation channels along with the original input. In both cases, the original layer configuration (kernel size, stride and padding type) remains the same except for its channel size being increased by $l$. All the other layers of the discriminator remain unchanged. %
When a new augmentation level is reached, one extra input channel of the filter is instantiated to process the bit $l+1$. 

These two ways of augmentation are beneficial as they make the checksum computation more challenging for the discriminator, i.e., making the discriminator unaware about the need of separating $\bvec{\vx}$ and $\bvec{\vs}$ from the concatenated input. 
We note that in order to take full advantage of the regularization effect of progressive augmentation, $\bvec{\vs}$ needs to be involved in the decision making process of the discriminator either through input or feature space augmentation. Augmenting $\bvec{\vs}$ with the output $D(\bvec{\vx})$ makes the task trivial, thereby disabling the regularization effect of the progressive augmentation.
In this work we only exploit $\bvec{\vs}$ by concatenating it with either the input or the hidden layers of the network. However, it is also possible to combine it with other image augmentation strategies, e.g. using $s$ as an indicator for the rotation angle, as in \cite{ChenSS2019}, or the type of color augmentation that is imposed on the input $\bvec{\vx}$ and encouraging $D$ to learn the type through the checksum principle.

\textbf{Progression scheduling.} \label{subsec:PA-schedule}
To schedule the progression we rely on the kernel inception distance (KID) introduced by~\cite{Binkowski2016MMDGAN} %
to decide if the performance of $G$ at the current augmentation level saturates or even starts degrading (typically happens when $D$ starts overfitting or becomes too powerful). Specifically, after $t$ discriminator iterations, %
we evaluate KID between synthetic samples and data samples drawn from the training set. %
If the current KID score is less than $5\%$ of the average of the two previous evaluations attained at the same augmentation level, the augmentation is leveled up, i.e. $l\rightarrow l+1$. To validate the effectiveness of this scheduling mechanism we exploit it for the learning rate adaptation as in~\cite{Binkowski2016MMDGAN} and compare it with progressive augmentation in the next section. %
\section{\label{sec:Experiments}Experiments}

\emph{\textbf{Datasets:}}  We consider four datasets: Fashion-MNIST~\cite{xiao2017}, CIFAR10~\cite{Cifar10_Krizhevsky09learningmultiple}, CELEBA-HQ $(128\times128)$~\cite{karras2018progressive} and Tiny-ImageNet (a simplified version of ImageNet~\cite{imagenet_cvpr09}), with the training set sizes equal to $60$\unit{k}, $50$\unit{k}, $27$\unit{k} and $100$\unit{k} plus the test set sizes equal to $10$\unit{k}, $10$\unit{k}, $3$\unit{k}, and $10$\unit{k}, respectively. Note that we focus on unsupervised image generation and do not use class label information.%

\emph{\textbf{Networks:}} 
We employ $\mathtt{SN\text{ }DCGAN}$~\cite{miyato2018spectral} and $\mathtt{SA\text{ }GAN}$~\cite{Zhang_SAGAN18}, both using spectral normalization (SN)~\cite{miyato2018spectral} in the discriminator for regularization. %
$\mathtt{SA\text{ }GAN}$ exploits the ResNet architecture with a self-attention (SA) layer~\cite{Zhang_SAGAN18}. Its generator additionally adopts self-modulation BN (sBN)~\cite{chen2018on} together with SN. We exploit the implementations provided by~\cite{Kurach2018GANlandscape,Zhang_SAGAN18}. %
Following~\cite{miyato2018spectral,Zhang_SAGAN18}, we train $\mathtt{SN\text{ }DCGAN}$ and $\mathtt{SA\text{ }GAN}$~\cite{Zhang_SAGAN18} with the non-saturation (NS) and hinge loss, respectively.

\emph{\textbf{Evaluation metrics:}} 
We use Fr{\'e}chet inception distance (FID)~\cite{FID} as the main evaluation metric. %
Additionally, we also report inception score (IS)~\cite{Theis2016a} and kernel inception distance (KID)~\cite{Binkowski2016MMDGAN} in Sec.~\ref{kid-is}. All measures are computed based on the same number of the test data samples and synthetic samples, following the evaluation framework of~\cite{LucicEqualGANs,Kurach2018GANlandscape}. By default all reported numbers correspond to the median of five independent runs with $300$\unit{k}, $500$\unit{k}, $400$\unit{k} and $500$\unit{k} training iterations for Fashion-MNIST, CIFAR10, CELEBA-HQ, and Tiny-ImageNet, respectively. 

\emph{\textbf{Training details:}} 
We use uniformly distributed noise vector $\bvec{\vz}\in[-1,1]^{128}$, the mini-batch size of $64$, and Adam optimizer~\cite{adamopt}. The two time-scale update rule (TTUR)~\cite{heuselttur2017} is considered when choosing the learning rates for $D$ and $G$. %
For progression scheduling KID\footnote{FID is used as the primary metric, KID is chosen for scheduling to avoid over-optimizing towards FID.} is evaluated using samples from the training set every $t=10$\unit{k} iterations, except for Tiny-ImageNet with $t=20$\unit{k} given its approximately $2\times$ larger training set. %
More details are provided in Sec.~\ref{sec:networks}. %

\subsection{PA Across Different Architectures and Datasets}\label{sec:PA-diff}

\begin{table}[t!]
	\vspace{-1em}
	\setlength{\tabcolsep}{0.2em} 
	\renewcommand{\arraystretch}{1.1}
	\centering
	\caption{FID improvement of $\mathtt{PA}$ across different datasets and network architectures. 
		We experiment with augmenting the input and feature spaces, see Sec.\ref{sec:PA-diff} for details.} \label{table:fid_overview} 
	\begin{tabular}{l|c|cccc|c} 
		\rowcolor{verylightgray}
		\footnotesize{}{\text{$\mathtt{Method}$}}	& \footnotesize{}{\text{$\mathtt{PA}$}}& 
		\footnotesize{}{\text{$\mathtt{F\text{-}MNIST}$}} & \footnotesize{}{\text{$\mathtt{CIFAR10}$}} & \footnotesize{}{\text{$\mathtt{CELEBA\text{-}HQ}$}} & \footnotesize{}{\text{$\mathtt{T\text{-}ImageNet}$}}&\footnotesize{}{\text{$\overline{\Delta\mathtt{PA}}$}} \tabularnewline

		\multirow{3}{*}{\text{$\mathtt{SN\text{ }DCGAN}$~\cite{miyato2018spectral}}} & \footnotesize{}{\xmark}& \text{\footnotesize$10.6$} & \text{\footnotesize$26.0$} & \text{\footnotesize$24.3$} & \text{\footnotesize-} & \multirow{3}{*}{\text{\footnotesize$4.2$}} \tabularnewline 	
		& \text{$\mathtt{input}$} & \text{\footnotesize$\mathbf{6.2}$} &  \text{\footnotesize$\mathbf{22.2}$} & \text{\footnotesize${20.8}$} & \text{\footnotesize-} \tabularnewline 
		& \text{$\mathtt{feat}$} & \text{\footnotesize$\mathbf{6.2}$} &  \text{\footnotesize${22.6}$} & \text{\footnotesize$\mathbf{18.8}$} & \text{\footnotesize-} \tabularnewline 
		
		\arrayrulecolor{gray}	\hline \arrayrulecolor{verylightgray}
		
		\multirow{3}{*}{\text{$\mathtt{SA\text{ }GAN\text{ }(sBN)}$~\cite{Zhang_SAGAN18}}}	& \footnotesize{}{\xmark}& \text{\footnotesize-} & \text{\footnotesize$18.8$} & \footnotesize{}{$17.8$} &\footnotesize{}{$47.6$} &\multirow{3}{*}{\text{\footnotesize$2.6$}}   \tabularnewline 
		& \text{$\mathtt{input}$} & \text{\footnotesize-}  & \text{\footnotesize$\mathbf{16.1}$} & \footnotesize{}{$\mathbf{15.4}$} & \text{\footnotesize$44.8$}  \tabularnewline 
		& \text{$\mathtt{feat}$} & \text{\footnotesize -}  & \text{\footnotesize$16.3$} & \footnotesize{}{$15.8$} & \text{\footnotesize$\mathbf{44.7}$} 
	\end{tabular}	
	\vspace{-0.5em}
\end{table}

Table~\ref{table:fid_overview} gives an overview of the FID performance achieved with and without applying the proposed progressive augmentation ($\mathtt{PA}$) across different datasets and networks. 
We observe consistent improvement of the FID score achieved by $\mathtt{PA}$ with both the input $\mathtt{PA}\text{ }(\mathtt{input})$ and feature $\mathtt{PA}\text{ }(\mathtt{feat})$ space augmentation (see Sec.~\ref{Ssubsec:abl aug level} for augmentation details and ablation study on the augmentation space). From $\mathtt{SN\text{ }DCGAN}$ to the ResNet-based $\mathtt{SA\text{ }GAN}$ the FID reduction preserves approximately around $3$ points, showing that the gain achieved by $\mathtt{PA}$ is complementary to the improvement on the architecture side. In comparison to input space augmentation, augmenting intermediate level features does not overly simplify the discriminator task, paralysing $\mathtt{PA}$. In the case of $\mathtt{SN\text{ }DCGAN}$ on CELEBA-HQ, it actually outperforms the input space augmentation. %
Overall, a stable performance gain of $\mathtt{PA}$, independent of the augmentation space choice, showcases high generalization quality of $\mathtt{PA}$ and its easy adaptation into different network designs.\footnote{We also experiment with using PA for WGAN-GP~\cite{Arjovsky2017WGAN}, improving FID from $25.0$ to $23.9$ on CIFAR10, see Sec.~\ref{Ssubsec:dropout} in the supp. material.} 

Lower FID values achieved by $\mathtt{PA}$ can be attributed mostly to %
the improved sample diversity. By looking at generated images in Fig.~\ref{fig:images} (and Fig.~\ref{Sfig:images} in the supp. material), we observe that $\mathtt{PA}$ increases the variation of samples while maintaining the same image fidelity. This is expected as $\mathtt{PA}$ being a regularizer does not modify the GAN architecture, as in PG-GAN~\cite{karras2018progressive} or BigGAN~\cite{Brock2019}, to directly improve the visual quality. Specifically, Fig.~\ref{fig:images} shows synthetic images produced by $\mathtt{SN\text{ }DCGAN}$ and $\mathtt{SA\text{ }GAN}$ with and without $\mathtt{PA}$, on Fashion-MNIST and CELEBA-HQ. By polar interpolation between two samples $\bvec{\vz}_{1}$ and $\bvec{\vz}_{2}$, from left to right we observe the clothes/gender change. $\mathtt{PA}$ improves sample variation, maintaining representative clothes/gender attributes and achieving smooth transition between samples (e.g. hair styles and facial expressions). For further evaluation, we also measure the diversity of generated samples with the MS-SSIM score \cite{odena17a}. We use $10$\unit{k} synthetic images generated with $\mathtt{SA\text{ }GAN}$ on CELEBA-HQ. Employing $\mathtt{PA}$ reduces MS-SSIM from $0.283$ to $0.266$, while  PG-GAN~\cite{karras2018progressive} achieves $0.283$, and MS-SSIM of $10$\unit{k} real samples is $0.263$.

\begin{figure*}
	\captionsetup[subfigure]{labelformat=empty}
	\vspace{0em}
	\centering
	\begin{subfigure}{0.48\textwidth}
		\includegraphics[width = \linewidth]{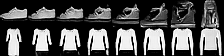}
		\caption{\text{$\mathtt{F\text{-}MNIST}$}: $\mathtt{SN\text{ }DCGAN}$ without $\mathtt{PA}$}
	\end{subfigure}
	\hfill
	\begin{subfigure}{0.48\textwidth}
		\includegraphics[width = \linewidth]{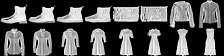} 
		\caption{\text{$\mathtt{F\text{-}MNIST}$}: $\mathtt{SN\text{ }DCGAN}$ with $\mathtt{PA}$}
	\end{subfigure}	
	
	\begin{subfigure}{0.8\textwidth}
		\vspace{1.0em}
		\includegraphics[width = \linewidth]{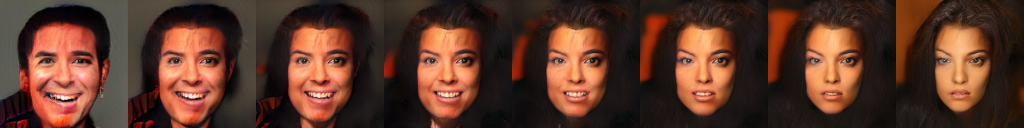}\\  
		\includegraphics[width = \linewidth]{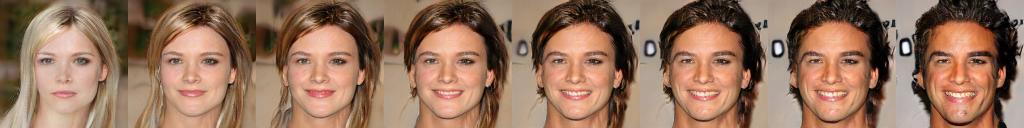} 
		\caption{\text{$\mathtt{CELEBA\text{-}HQ}$}: $\mathtt{SA\text{ }GAN}$ without $\mathtt{PA}$}
	\end{subfigure}
	\begin{subfigure}{0.8\textwidth}
		\includegraphics[width = \linewidth]{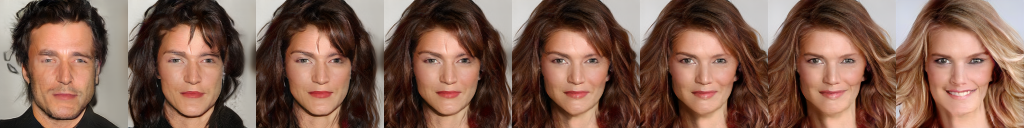}\\
		\includegraphics[width = \linewidth]{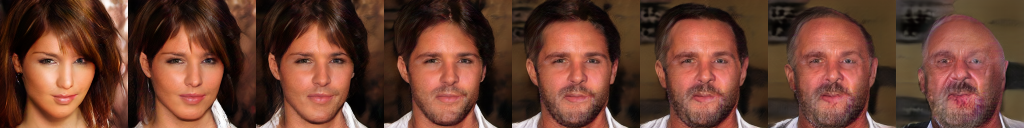} 
		\caption{\text{$\mathtt{CELEBA\text{-}HQ}$}: $\mathtt{SA\text{ }GAN}$ with $\mathtt{PA}$}
	\end{subfigure}	
	
	\caption{Synthetic images generated through latent space interpolation with and without using $\mathtt{PA}$. $\mathtt{PA}$ helps to improve variation across interpolated samples, i.e., no close-by images looks alike.}	\label{fig:images}
	\vspace{-1.6em}
\end{figure*}

\textbf{Comparison with SotA on Human Face Synthesis.} Deviating from from low- to high-resolution human face synthesis, the recent work COCO-GAN~\cite{lin2019cocogan} outperformed PG-GAN~\cite{karras2018progressive} on the CELEBA dataset~\cite{Liu_Celeba} via conditional coordinating. At the resolution $64$ of CELEBA, $\mathtt{PA}$ improves the $\mathtt{SA\text{ }GAN}$ FID from $4.11$ to $3.35$, being better than COCO-GAN, which achieves FID of $4.0$ and outperforms PG-GAN at the resolution $128$ (FID of $5.74$ vs. $7.30$).
Thus we conclude that the quality of samples generated by $\mathtt{PA}$ is comparable to the quality of samples generated by the recent state-of-the-art models~\cite{lin2019cocogan,karras2018progressive} on human face synthesis.

\textbf{Ablation Study.}
In Fig.~\ref{fig:fig_abl_pa_a} and Table~\ref{table:abla_pa_b} we present an ablation study on $\mathtt{PA}$, comparing single-level augmentation (without progression) with progressive multi-level $\mathtt{PA}$, showing the benefit of progression. From no augmentation to the first level augmentation, the required number of iterations varies over the datasets and architectures ($30$\unit{k}$\sim70$\unit{k}). %
Generally the number of reached augmentation levels is less than $15$. %
Fig.~\ref{fig:fig_abl_pa_a} also shows that single-level augmentation already improves the performance over the baseline $\mathtt{SN\text{ }DCGAN}$. However, the standard deviation of its FIDs across five independent runs starts increasing at later iterations. By means of progression, we can counteract this instability, while reaching a better FID result. Table~\ref{table:abla_pa_b} further compares augmentation at different levels with and without continuing with progression. Both augmentation and progression are beneficial, while progression alleviates the need of case dependent tuning of the augmentation level.

As a generic mechanism to monitor the GAN training, progression scheduling is usable not only for augmentation level-up, but also for other hyperparameter adaptations over iterations. Analogous to \cite{Binkowski2016MMDGAN} here we test it for the learning rate adaptation. %
From Fig.~\ref{fig:fig_abl_pa_a}, progression scheduling shows its effectiveness in assisting both the learning rate adaptation and $\mathtt{PA}$ for an improved FID performance. $\mathtt{PA}$ outperforms learning rate adaptation, i.e. median FID $22.2$ vs. $24.0$ across five independent runs.

\textbf{Regularization Effect of PA.}\label{sec_toy_example}
Fig.~\ref{fig:toy_d_g_fid_loss} depicts the discriminator loss ($D$ loss) and the generator loss ($G$ loss) behaviour as well as the FID curves over iterations. It shows that the discriminator of $\mathtt{SN\text{ }DCGAN}$ very quickly becomes over-confident, providing a non-informative backpropagation signal to train the generator and thus leading to the increase of the $G$ loss. %
$\mathtt{PA}$ has a long lasting regularization effect on $\mathtt{SN\text{ }DCGAN}$ by means of progression and helps to maintain a healthy competition between its discriminator and generator. 
Each rise of the $D$ loss and drop of the $G$ loss coincides with an iteration at which the augmentation level increases, and then gradually reduces after the discriminator timely adapts to the new bit.
Observing the behaviour of the $D$ and $G$ losses, we conclude that both $\mathtt{PA}\text{ }(\mathtt{input})$ and $\mathtt{PA}\text{ }(\mathtt{feat})$ can effectively prevent the $\mathtt{SN\text{ }DCGAN}$ discriminator from overfitting, alleviating
the vanishing gradient issue and thus enabling continuous learning of the generator.
At the level one augmentation, both $\mathtt{PA}\text{ }(\mathtt{feat})$ and $\mathtt{PA}\text{ }(\mathtt{input})$ start from the similar overfitting stage, i.e., $(a)$ and $(b)$ respectively at the iteration $60$\unit{k} and $70$\unit{k}. %
Combining the bit $s$ directly with high-level features eases the checksum computation. As a result, the $D$ loss of $\mathtt{PA}\text{ }(\mathtt{feat_{N/8}})$ reduces faster, but making its future task more difficult due to overfitting to the previous augmentation level. On the other hand, $\mathtt{PA}\text{ }(\mathtt{input})$ let the bits pass through all layers, and thus its adaptation to augmentation progression improves over iterations. In the end, both $\mathtt{PA}\text{ }(\mathtt{feat})$ and $\mathtt{PA}\text{ }(\mathtt{input})$ lead to similar regularization effect and result in the improved FID scores.

\begin{figure}[t!]
	\centering
	\begin{minipage}{0.5\textwidth}
		\centering
		\includegraphics[width=\linewidth]{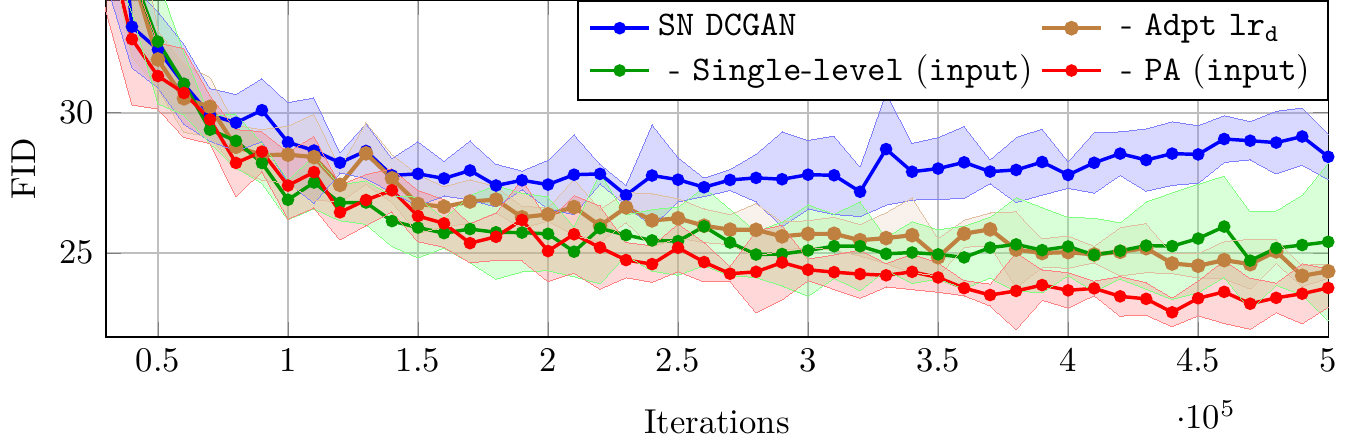}
		\caption{FID learning curves on $\mathtt{SN\text{ }DCGAN}$ CIFAR10. The curves show the mean FID with one standard deviation across five random runs.}\label{fig:fig_abl_pa_a}
	\end{minipage}\hfill
	\begin{minipage}{0.45\textwidth}
		\setlength{\tabcolsep}{0.36em} 
		\renewcommand{\arraystretch}{1.0}
		\centering
		\captionof{table}{Median FIDs of input space augmentation starting from the level $l$ with and without progression on CIFAR10 with $\mathtt{SN\text{ }DCGAN}$.} \label{table:abla_pa_b}
		\begin{tabular}{c|cc|c} 
			\rowcolor{verylightgray}
			\footnotesize{}{\text{$\mathtt{Augment.\text{ }level}$}} & \multicolumn{2}{c}{\footnotesize{}{\text{$\mathtt{Progression}$}}}   & \tabularnewline 
			\rowcolor{verylightgray}
			\text{\footnotesize$l$}   & 	\footnotesize{}{\text{\xmark }}	& \footnotesize{}{\text{\cmark }} & \multirow{-2}{*}{\footnotesize{}{\text{$\Delta\mathtt{PA}$}}} \tabularnewline 
			\text{\footnotesize$0$} & \text{\footnotesize$26.0$} & \text{\footnotesize$\mathbf{22.2}$} & \text{\footnotesize$3.8$} \tabularnewline 
			\arrayrulecolor{verylightgray}	\hline 
			\text{\footnotesize$1$}&\text{\footnotesize$23.8$}	& \text{\footnotesize$22.3$} & \text{\footnotesize$1.5$} \tabularnewline 
			\arrayrulecolor{verylightgray}	\hline 
			\text{\footnotesize$2$} &\text{\footnotesize$23.6$} & \text{\footnotesize$22.9$}&  \text{\footnotesize$0.7$} \tabularnewline 
			\arrayrulecolor{verylightgray}	\hline 
			
			\text{\footnotesize$3$} &\text{\footnotesize$23.5$} & \text{\footnotesize$22.9$}&  \text{\footnotesize$0.6$} \tabularnewline 
			\arrayrulecolor{verylightgray}	\hline 
			
			\text{\footnotesize$4$} &  \text{\footnotesize$23.5$} & \text{\footnotesize$23.2$} 		&	\text{\footnotesize$0.3$}
		\end{tabular}			
	\end{minipage}	
\end{figure}

In Fig.~\ref{fig:toy_d_g_fid_loss} we also evaluate the $\mathtt{Dropout}$~\cite{JMLR:v15:srivastava14a} regularization applied on the fourth convolutional layer with the keep rate $0.7$ (the best performing setting in our experiments). %
Both $\mathtt{Dropout}$ and $\mathtt{PA}$ resort to random variables for regularization. The former randomly removes features, while the latter augments them with additional random bits and adjusts accordingly the class label. %
In contrast to $\mathtt{Dropout}$, $\mathtt{PA}$ has a stronger regularization effect and leads to faster convergence (more rapid reduction of FID scores). %
In addition, we compare $\mathtt{PA}$ with the $\mathtt{Reinit.}$ baseline, where at each scheduled progression all weights are reinitialized with Xavier initialization \cite{GlorotAISTATS2010}. Compared to $\mathtt{PA}$, using $\mathtt{Reinit.}$ strategy leads to longer adaptation time (the $D$ loss decay is much slower) and oscillatory GAN behaviour, thus resulting in dramatic fluctuations of FID scores over iterations.

\begin{figure}
	\centering
	\vspace{0em}
		\begin{minipage}{\textwidth}
	\begin{subfigure}{\textwidth}
		\centering
		\includegraphics[width=\linewidth]{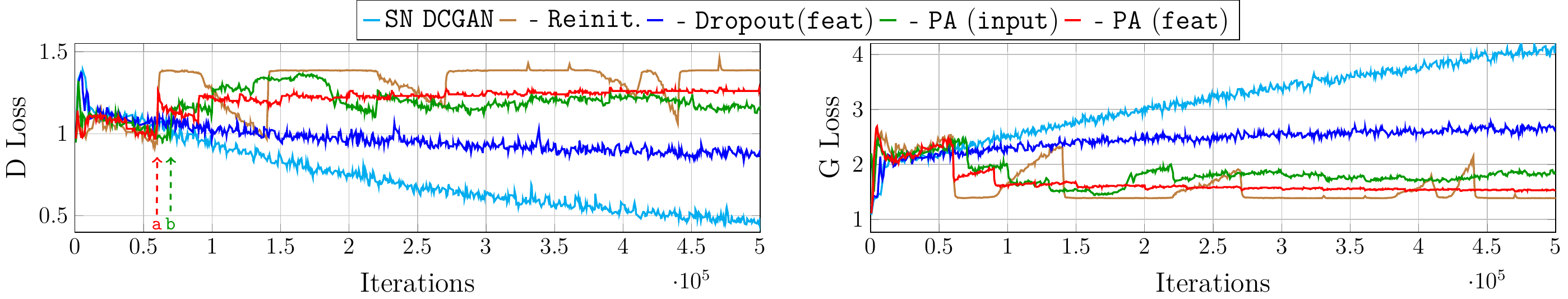}
		\caption{Discriminator ($D$) and generator ($G$) loss over iterations}\label{fig:fig_dg_loss}
	\end{subfigure}
	\end{minipage}
	\begin{minipage}{.5\textwidth}
	\begin{subfigure}{\textwidth}	
		\vspace{0.5em}
		\centering
		\includegraphics[width=\linewidth]{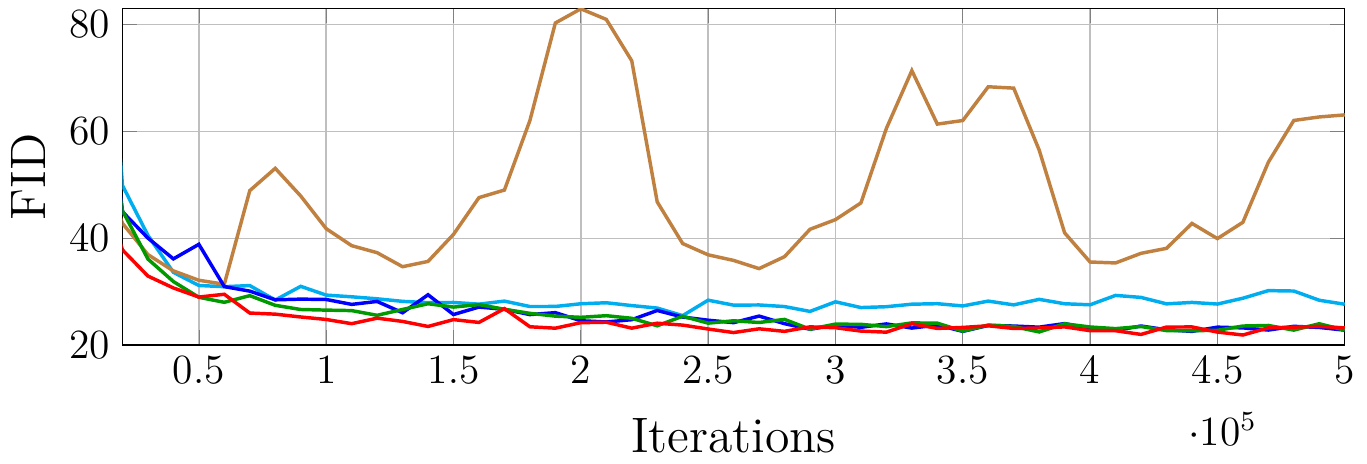}
		\caption{FID over iterations}\label{fig:fig_fid}
	\end{subfigure}
	\end{minipage}
	\hspace{.6em}
	\begin{minipage}{.45\textwidth}
	\caption{\label{fig:toy_d_g_fid_loss} Behaviour of the discriminator loss ($D$ loss) and the generator loss ($G$ loss) as well as FID changes over iterations, using SN DCGAN on CIFAR10. $\mathtt{PA}$ acts as a stochastic regularizer, preventing the discriminator from becoming overconfident.}
	\end{minipage}	
\end{figure}

\subsection{\label{subsec:Exp-Regulariz}Comparison and Combination with Other Regularizers}
We further compare and combine $\mathtt{PA}$ with other regularization techniques, i.e., one-sided label smoothing~\cite{SalimansNIPS2016}, $\mathtt{GP}$ from~\cite{gulrajani_NIPS2017}, its zero-centered alternative $\mathtt{GP_{zero\text{-}cent}}$ from~\cite{Roth_NIPS2017}, $\mathtt{Dropout}$~\cite{JMLR:v15:srivastava14a}, and self-supervised GAN training via auxiliary rotation loss ($\mathtt{SS}$)~\cite{ChenSS2019}. 

One-sided label smoothing ($\mathtt{Label\text{ }smooth.}$) weakens the discriminator by smoothing its decision boundary, i.e., changing the positive labels from one to a smaller value. This is analogous to introducing label noise for the data samples, whereas $\mathtt{PA}$ alters the target labels based on the deterministic checksum principle. %
Benefiting from a smoothed decision boundary, $\mathtt{Label\text{ }smooth.}$ slightly improves the performance of $\mathtt{SN\text{ }DCGAN}$ ($26.0$ vs. $25.8$), but underperforms in comparison to $\mathtt{PA}\text{ }(\mathtt{input})$ ($22.2$) and $\mathtt{PA}\text{ }(\mathtt{feat})$ ($22.6$). By applying $\mathtt{PA}\text{ }$ on top of $\mathtt{Label\text{ }smooth.}$ we observe a similar reduction of the FID score ($23.1$ and $22.3$ for input and feature space augmentation, respectively).

Both $\mathtt{GP}$ and $\mathtt{GP_{zero\text{-}cent}}$ regularize the norms of gradients to stabilize the GAN training. The former aims at a 1-Lipschitz discriminator, and the latter is a closed-form approximation of adding input noise. %
Table~\ref{table:pa_gp_dropout} shows that both of them are compatible with $\mathtt{PA}$ but degrade the performance of $\mathtt{SN\text{ }DCGAN}$ alone and its combination with $\mathtt{PA}$. This effect has been also observed in~\cite{Kurach2018GANlandscape,Brock2019}, constraining the learning of the discriminator improves the GAN training stability but at the cost of performance degradation. Note that, however, with $\mathtt{PA}$ performance degradation is smaller.

$\mathtt{Dropout}$ shares a common stochastic nature with $\mathtt{PA}$ as illustrated in Fig.~\ref{fig:toy_d_g_fid_loss} and in the supp. material. %
We observe from Table~\ref{table:pa_gp_dropout} that $\mathtt{Dropout}$ and $\mathtt{PA}$ can be both exploited as effective regularizers. $\mathtt{Dropout}$ acts locally on the layer. The layer outputs are randomly and independently subsampled, thinning the network. In contrast, $\mathtt{PA}$ augments the input or the layer with extra channels containing random bits, these bits also change the class label of the input and thus alter the network decision process. $\mathtt{Dropout}$ helps to break-up situations where the layer co-adapts to correct errors from prior layers and enables the network to timely re-learn features of constantly changing synthetic samples. $\mathtt{PA}$ regularizes the decision process of $D$, forcing $D$ to comprehend the input together with the random bits for correct classification and has stronger regularization effect than $\mathtt{Dropout}$, see Fig.~\ref{fig:toy_d_g_fid_loss} and the supp. material. Hence, they have different roles. Their combination further improves FID by $\sim 0.8$ point on average, showing the complementarity of both approaches. It is worth noting that $\mathtt{Dropout}$ is sensitive to the selection of the layer at which it is applied. In our experiments (see the supp. material) it performs best when applied at the fourth convolutional layer.%

\begin{table}[t!]
	\setlength{\tabcolsep}{0.2em} 
	\renewcommand{\arraystretch}{1.1}
	\centering
	\caption{FID performance of $\mathtt{PA}$, different regularization techniques and their combinations on CIFAR10, see Sec.~\ref{subsec:Exp-Regulariz} for details.} \label{table:pa_gp_dropout}
	\begin{tabular}{l|c|c|ccccc|c} 
		\rowcolor{verylightgray}
		& & & \footnotesize{}{\text{-$\mathtt{Label\text{ }smooth.}$}}  &\footnotesize{}{\text{-$\mathtt{GP}$}} & \footnotesize{}{\text{-$\mathtt{GP_{zero\text{-}cent}}$}} & \footnotesize{}{\text{-$\mathtt{Dropout}$}} &\footnotesize{}{\text{-$\mathtt{SS}$}}  & \tabularnewline 
		\rowcolor{verylightgray}
		\multirow{-2}{*}{	\footnotesize{}{\text{$\mathtt{Method}$}}} & \multirow{-2}{*}{\footnotesize{}{\text{$\mathtt{PA}$}}} & \multirow{-2}{*}{\footnotesize{}{\text{$\mathtt{GAN}$}}} & 	\footnotesize{}{\text{\cite{SalimansNIPS2016}}} &  \footnotesize{}{\cite{gulrajani_NIPS2017}} & \footnotesize{}{ \cite{Roth_NIPS2017}} &
		\footnotesize{}{\cite{JMLR:v15:srivastava14a}} & \footnotesize{}{\cite{ChenSS2019}} &  \multirow{-2}{*}{\footnotesize{}{\text{$\overline{\Delta\mathtt{PA}}$}}} \tabularnewline  
		\multirow{2}{*}{\footnotesize{}{\text{$\mathtt{SN\text{ }DCGAN}$~\cite{miyato2018spectral}}}} & 	   	\footnotesize{}{\text{\xmark }} & \text{\footnotesize${26.0}$}  & \text{\footnotesize${25.8}$} & \text{\footnotesize${26.7}$} &\text{\footnotesize${26.5}$}  & \text{\footnotesize${22.1}$} &  \text{\footnotesize${-}$}  &	\tabularnewline  
			& \footnotesize{}{\text{$\mathtt{input}$}} & \text{\footnotesize${22.2}$}  & \text{\footnotesize$23.1$}  & \text{\footnotesize${21.8}$} & \text{\footnotesize${22.3}$} & \text{\footnotesize${21.9}$} & \text{\footnotesize${-}$} & \text{\footnotesize${3.0}$} \tabularnewline 
		& \footnotesize{}{\text{$\mathtt{feat}$}} & \text{\footnotesize${22.6}$}  & \text{\footnotesize${22.3}$}  & \text{\footnotesize${22.7}$} & \text{\footnotesize${23.0}$} & \text{\footnotesize${\mathbf{20.6}}$} &\text{\footnotesize${-}$}  &
	\footnotesize${{3.1}}$\tabularnewline 
		\arrayrulecolor{verylightgray}	\hline 
		\multirow{2}{*}{\footnotesize{}{\text{$\mathtt{SA\text{ }GAN\text{ }(sBN)}$~\cite{Zhang_SAGAN18}}}}	 & 	   	\footnotesize{}{\text{\xmark }} & \text{\footnotesize${18.8}$}  & \text{\footnotesize$-$} & \text{\footnotesize${17.8}$} &\text{\footnotesize${17.8}$} & \text{\footnotesize${16.2}$} &\text{\footnotesize${15.7}$} & \tabularnewline  
		& \footnotesize{}{\text{$\mathtt{input}$}} & \text{\footnotesize${16.1}$}  & \text{\footnotesize${-}$}  & \text{\footnotesize${15.8}$} & \text{\footnotesize${16.1}$} & \text{\footnotesize${15.5}$} &\text{\footnotesize$\mathbf{14.7}$}  & \text{\footnotesize${1.3}$}	\tabularnewline
		& \footnotesize{}{\text{$\mathtt{feat}$}} & \text{\footnotesize${16.3}$}  & \text{\footnotesize${-}$}  & \text{\footnotesize${16.1}$} & \text{\footnotesize${15.9}$} &  \text{\footnotesize${15.6}$} &\text{\footnotesize$14.9$} &
		\footnotesize$1.3$\tabularnewline 
		\arrayrulecolor{verylightgray}\hline 
		& \footnotesize{}{\text{$\overline{\Delta\mathtt{PA}}$}} & \text{\footnotesize${3.1}$} & \text{\footnotesize${3.1}$}& \text{\footnotesize${3.2}$} & \text{\footnotesize${2.8}$} & \text{\footnotesize${0.8}$} &  \text{\footnotesize${0.9}$} &\text{\footnotesize${2.3}$}\tabularnewline 
	\end{tabular}	
	\vspace{-1em}
\end{table}

Self-supervised training (SS-GAN) in~\cite{ChenSS2019} regularizes the discriminator by encouraging it to solve an auxiliary image rotation prediction task. From the perspective of self-supervision, $\mathtt{PA}$ presents the discriminator a checksum computation task, whereas telling apart the data and synthetic samples becomes a sub-task. Rotation prediction task was initially proposed and found useful in~\cite{Gidaris2018} to improve feature learning of convolutional networks. The checksum principle is derived from Theorem~\ref{thm}. Their combination is beneficial and achieves the best FID of $14.7$ for the unsupervised setting on CIFAR10, which is the same score as in the supervised case with large scale BigGAN training~\cite{Brock2019}. 

Overall, we observe that $\mathtt{PA}$ is consistently beneficial when combining with other regularization techniques, independent of input or feature space augmentation. Additional improvement of the FID score can come along with fine selection of the augmentation space type.

\section{\label{sec:Conclusion}Conclusion}
In this work we have proposed progressive augmentation (PA) -  a novel regularization method for GANs. %
Different to standard data augmentation our approach does not modify the training samples, instead it progressively augments them or their feature maps with auxiliary random bits and casts the discrimination task into the checksum computation. %
PA helps to entangle the discriminator and thus to avoid its early performance saturation. %
We experimentally have shown consistent performance improvements of employing PA-GAN across multiple benchmarks and demonstrated that PA generalizes well across different network architectures and is complementary to other regularization techniques. Apart from generative modelling, as a future work we are interested in exploiting PA for semi-supervised learning, generative latent modelling and transfer learning.


\bibliography{references} 

\begin{thebibliography}{37}
\providecommand{\natexlab}[1]{#1}
\providecommand{\url}[1]{\texttt{#1}}
\expandafter\ifx\csname urlstyle\endcsname\relax
  \providecommand{\doi}[1]{doi: #1}\else
  \providecommand{\doi}{doi: \begingroup \urlstyle{rm}\Url}\fi

\bibitem[Arjovsky \& Bottou(2017)Arjovsky and Bottou]{Arjovsky2017TowardsPM}
Martin Arjovsky and L{\'e}on Bottou.
\newblock Towards principled methods for training generative adversarial
  networks.
\newblock In \emph{International Conference on Learning Representations
  (ICLR)}, 2017.

\bibitem[Arjovsky et~al.(2017)Arjovsky, Chintala, and Bottou]{Arjovsky2017WGAN}
Martin Arjovsky, Soumith Chintala, and L{\'e}on Bottou.
\newblock {W}asserstein generative adversarial networks.
\newblock In \emph{Advances in Neural Information Processing Systems (NIPS)},
  2017.

\bibitem[Bi{\'n}kowski et~al.(2018)Bi{\'n}kowski, Sutherland, Arbel, and
  Gretton]{Binkowski2016MMDGAN}
Miko{\l}aj Bi{\'n}kowski, Dougal~J. Sutherland, Michael~N. Arbel, and Athur
  Gretton.
\newblock Demystifying {MMD GAN}s.
\newblock In \emph{International Conference on Learning Representations
  (ICLR)}, 2018.

\bibitem[Brock et~al.(2019)Brock, Donahue, and Simonyan]{Brock2019}
Andrew Brock, Jeff Donahue, and Karen Simonyan.
\newblock Large scale {GAN} training for high fidelity natural image synthesis.
\newblock In \emph{International Conference on Learning Representations
  (ICLR)}, 2019.

\bibitem[Chen et~al.(2019{\natexlab{a}})Chen, Lucic, Houlsby, and
  Gelly]{chen2018on}
Ting Chen, Mario Lucic, Neil Houlsby, and Sylvain Gelly.
\newblock On self modulation for generative adversarial networks.
\newblock In \emph{International Conference on Learning Representations
  (ICLR)}, 2019{\natexlab{a}}.

\bibitem[Chen et~al.(2019{\natexlab{b}})Chen, Zhai, Ritter, Lucic, and
  Houlsby]{ChenSS2019}
Ting Chen, Xiaohua Zhai, Marvin Ritter, Mario Lucic, and Neil Houlsby.
\newblock Self-supervised gans via auxiliary rotation loss.
\newblock In \emph{IEEE Conference on Computer Vision and Pattern Recognition
  (CVPR)}, 2019{\natexlab{b}}.

\bibitem[Deng et~al.(2009)Deng, Dong, Socher, Li, Li, and
  Fei-Fei]{imagenet_cvpr09}
J.~Deng, W.~Dong, R.~Socher, L.-J. Li, K.~Li, and L.~Fei-Fei.
\newblock {ImageNet: A Large-Scale Hierarchical Image Database}.
\newblock In \emph{IEEE Conference on Computer Vision and Pattern Recognition
  (CVPR)}, 2009.

\bibitem[Donahue et~al.(2017)Donahue, Kr{\"a}henb{\"u}hl, and Darrell]{ali2017}
Jeff Donahue, Philipp Kr{\"a}henb{\"u}hl, and Trevor Darrell.
\newblock Adversarial feature learning.
\newblock In \emph{International Conference on Learning Representations
  (ICLR)}, 2017.

\bibitem[Durugkar et~al.(2017)Durugkar, Gemp, and
  Mahadevan]{Durugkar2016GenerativeMN}
Ishan~P. Durugkar, Ian Gemp, and Sridhar Mahadevan.
\newblock Generative multi-adversarial networks.
\newblock In \emph{International Conference on Learning Representations
  (ICLR)}, 2017.

\bibitem[Gidaris et~al.(2018)Gidaris, Singh, and Komodakis]{Gidaris2018}
Spyros Gidaris, Praveer Singh, and Nikos Komodakis.
\newblock Unsupervised representation learning by predicting image rotations.
\newblock In \emph{International Conference on Learning Representations
  (ICLR)}, Vancouver, Canada, April 2018.

\bibitem[Glorot \& Bengio(2010)Glorot and Bengio]{GlorotAISTATS2010}
Xavier Glorot and Yoshua Bengio.
\newblock Understanding the difficulty of training deep feedforward neural
  networks.
\newblock In \emph{Proceedings of the Thirteenth International Conference on
  Artificial Intelligence and Statistics (AISTATS 2010)}, 2010.

\bibitem[Goodfellow et~al.(2014)Goodfellow, Pouget-Abadie, Mirza, Xu,
  Warde-Farley, Ozair, Courville, and Bengio]{goodfellow2014generative}
Ian Goodfellow, Jean Pouget-Abadie, Mehdi Mirza, Bing Xu, David Warde-Farley,
  Sherjil Ozair, Aaron Courville, and Yoshua Bengio.
\newblock Generative adversarial nets.
\newblock In \emph{Advances in Neural Information Processing Systems (NIPS)},
  2014.

\bibitem[Gulrajani et~al.(2017)Gulrajani, Ahmed, Arjovsky, Dumoulin, and
  Courville]{gulrajani_NIPS2017}
Ishaan Gulrajani, Faruk Ahmed, Martin Arjovsky, Vincent Dumoulin, and Aaron~C
  Courville.
\newblock Improved training of {Wasserstein} {GAN}s.
\newblock In \emph{Advances in Neural Information Processing Systems (NIPS)},
  2017.

\bibitem[Heusel et~al.(2017)Heusel, Ramsauer, Unterthiner, Nessler, and
  Hochreiter]{heuselttur2017}
Martin Heusel, Hubert Ramsauer, Thomas Unterthiner, Bernhard Nessler, and Sepp
  Hochreiter.
\newblock {GANs} trained by a two time-scale update rule converge to a local
  nash equilibrium.
\newblock In \emph{Advances in Neural Information Processing Systems (NIPS)},
  2017.

\bibitem[Husz{\'a}r(2015)]{FID}
Ferenc Husz{\'a}r.
\newblock How (not) to train your generative model: Scheduled sampling,
  likelihood, adversary?
\newblock \emph{arXiv: 1511.05101}, 2015.

\bibitem[Karras et~al.(2018)Karras, Aila, Laine, and
  Lehtinen]{karras2018progressive}
Tero Karras, Timo Aila, Samuli Laine, and Jaakko Lehtinen.
\newblock Progressive growing of {GAN}s for improved quality, stability, and
  variation.
\newblock In \emph{International Conference on Learning Representations
  (ICLR)}, 2018.

\bibitem[Kingma \& Ba(2015)Kingma and Ba]{adamopt}
Diederik~P. Kingma and Jimmy Ba.
\newblock Adam: {A} method for stochastic optimization.
\newblock In \emph{International Conference on Learning Representations
  (ICLR)}, 2015.

\bibitem[Krizhevsky(2009)]{Cifar10_Krizhevsky09learningmultiple}
Alex Krizhevsky.
\newblock Learning multiple layers of features from tiny images.
\newblock Technical report, 2009.

\bibitem[Kurach et~al.(2018)Kurach, Lu\u{c}i\'{c}, Zhai, Michalski, and
  Gelly]{Kurach2018GANlandscape}
Karol Kurach, Mario Lu\u{c}i\'{c}, Xiaohua Zhai, Marcin Michalski, and Sylvain
  Gelly.
\newblock The {GAN} landscape: Losses, architectures, regularization, and
  normalization.
\newblock \emph{arXiv: 1807.04720}, 2018.

\bibitem[Lin et~al.(2019)Lin, Chang, Chen, Juan, Wei, and Chen]{lin2019cocogan}
Chieh~Hubert Lin, Chia{-}Che Chang, Yu{-}Sheng Chen, Da{-}Cheng Juan, Wei Wei,
  and Hwann{-}Tzong Chen.
\newblock {COCO-GAN:} generation by parts via conditional coordinating.
\newblock In \emph{IEEE International Conference on Computer Vision (ICCV)},
  2019.

\bibitem[Liu et~al.(2015)Liu, Luo, Wang, and Tang]{Liu_Celeba}
Ziwei Liu, Ping Luo, Xiaogang Wang, and Xiaoou Tang.
\newblock Deep learning face attributes in the wild.
\newblock In \emph{International Conference on Computer Vision (ICCV)}, 2015.

\bibitem[Lu\u{c}i\'{c} et~al.(2018)Lu\u{c}i\'{c}, Kurach, Michalski, Gelly, and
  Bousquet]{LucicEqualGANs}
Mario Lu\u{c}i\'{c}, Karol Kurach, Marcin Michalski, Sylvain Gelly, and Olivier
  Bousquet.
\newblock Are {GAN}s created equal? {A} large-scale study.
\newblock In \emph{Advances in Neural Information Processing Systems (NIPS)},
  2018.

\bibitem[Mao et~al.(2016)Mao, Li, Xie, Lau, and Wang]{MaoLXLW16}
Xudong Mao, Qing Li, Haoran Xie, Raymond Y.~K. Lau, and Zhen Wang.
\newblock Multi-class generative adversarial networks with the {L2} loss
  function.
\newblock \emph{arXiv:1611.04076}, 2016.

\bibitem[Mescheder et~al.(2018)Mescheder, Geiger, and
  Nowozin]{MeschederICML2018}
Lars Mescheder, Andreas Geiger, and Sebastian Nowozin.
\newblock Which training methods for {GAN}s do actually converge?
\newblock In \emph{International Conference on Machine learning (ICML)}, 2018.

\bibitem[Miyato et~al.(2018)Miyato, Kataoka, Koyama, and
  Yoshida]{miyato2018spectral}
Takeru Miyato, Toshiki Kataoka, Masanori Koyama, and Yuichi Yoshida.
\newblock Spectral normalization for generative adversarial networks.
\newblock In \emph{International Conference on Learning Representations
  (ICLR)}, 2018.

\bibitem[Nowozin et~al.(2016)Nowozin, Cseke, and Tomioka]{Nowozin2016fGANTG}
Sebastian Nowozin, Botond Cseke, and Ryota Tomioka.
\newblock {f-GAN}: Training generative neural samplers using variational
  divergence minimization.
\newblock In \emph{Advances in Neural Information Processing Systems (NIPS)},
  2016.

\bibitem[Odena et~al.(2017)Odena, Olah, and Shlens]{odena17a}
Augustus Odena, Christopher Olah, and Jonathon Shlens.
\newblock Conditional image synthesis with auxiliary classifier {GAN}s.
\newblock In \emph{International Conference on Learning Representations
  (ICLR)}, 2017.

\bibitem[Radford et~al.(2016)Radford, Metz, and
  Chintala]{Radford2016UnsupervisedRL}
Alec Radford, Luke Metz, and Soumith Chintala.
\newblock Unsupervised representation learning with deep convolutional
  generative adversarial networks.
\newblock In \emph{International Conference on Learning Representations
  (ICLR)}, 2016.

\bibitem[Roth et~al.(2017)Roth, Lucchi, Nowozin, and Hofmann]{Roth_NIPS2017}
Kevin Roth, Aurelien Lucchi, Sebastian Nowozin, and Thomas Hofmann.
\newblock Stabilizing training of generative adversarial networks through
  regularization.
\newblock In \emph{Advances in Neural Information Processing Systems (NIPS)},
  2017.

\bibitem[Sajjadi et~al.(2018)Sajjadi, Parascandolo, and
  Arash~Mehrjou]{SajParMehSch18}
Mehdi S.~M. Sajjadi, Giambattista Parascandolo, and Bernhard~Sch{\"o}lkopf
  Arash~Mehrjou.
\newblock Tempered adversarial networks.
\newblock In \emph{International Conference on Machine Learning (ICML)}, 2018.

\bibitem[Salimans et~al.(2016)Salimans, Goodfellow, Zaremba, Cheung, Radford,
  Chen, and Chen]{SalimansNIPS2016}
Tim Salimans, Ian Goodfellow, Wojciech Zaremba, Vicki Cheung, Alec Radford,
  Xi~Chen, and Xi~Chen.
\newblock Improved techniques for training {GAN}s.
\newblock In \emph{Advances in Neural Information Processing Systems (NIPS)},
  2016.

\bibitem[S{\o}nderby et~al.(2017)S{\o}nderby, Caballero, Theis, Shi, and
  Husz{\'a}r]{Sonderby2016AmortisedMI}
Casper~Kaae S{\o}nderby, Jose Caballero, Lucas Theis, Wenzhe Shi, and Ferenc
  Husz{\'a}r.
\newblock Amortised map inference for image super-resolution.
\newblock In \emph{International Conference on Learning Representations
  (ICLR)}, 2017.

\bibitem[Srivastava et~al.(2014)Srivastava, Hinton, Krizhevsky, Sutskever, and
  Salakhutdinov]{JMLR:v15:srivastava14a}
Nitish Srivastava, Geoffrey Hinton, Alex Krizhevsky, Ilya Sutskever, and Ruslan
  Salakhutdinov.
\newblock Dropout: A simple way to prevent neural networks from overfitting.
\newblock \emph{Journal of Machine Learning Research (JMLR)}, 2014.

\bibitem[Theis et~al.(2016)Theis, van~den Oord, and Bethge]{Theis2016a}
Lucas Theis, Aaron van~den Oord, and Matthias Bethge.
\newblock A note on the evaluation of generative models.
\newblock In \emph{International Conference on Learning Representations
  (ICLR)}, 2016.

\bibitem[Tompson et~al.(2015)Tompson, Goroshin, Jain, LeCun, and
  Bregler]{Tompson2015EfficientOL}
Jonathan Tompson, Ross Goroshin, Arjun Jain, Yann LeCun, and Christoph Bregler.
\newblock Efficient object localization using convolutional networks.
\newblock In \emph{IEEE Conference on Computer Vision and Pattern Recognition
  (CVPR)}, 2015.

\bibitem[Xiao et~al.(2017)Xiao, Rasul, and Vollgraf]{xiao2017}
Han Xiao, Kashif Rasul, and Roland Vollgraf.
\newblock Fashion-mnist: a novel image dataset for benchmarking machine
  learning algorithms.
\newblock \emph{arXiv: 1708.07747}, 2017.

\bibitem[Zhang et~al.(2018)Zhang, Goodfellow, Metaxas, and
  Odena]{Zhang_SAGAN18}
Han Zhang, Ian~J. Goodfellow, Dimitris~N. Metaxas, and Augustus Odena.
\newblock Self-attention generative adversarial networks.
\newblock \emph{arXiv: 1805.08318}, 2018.

\end{thebibliography}
\bibliographystyle{../../iclr2019_template/iclr2019_conference}
\newpage
\begin{center}
	\setlength{\parindent}{0pt}
	\setlength{\parskip}{0pt}
	\rule{\linewidth}{2pt} 
	\\ [4pt]
	\textbf{\Large Supplemental Materials}: \\[4pt]
	\Large	Progressive Augmentation of GANs
	\rule{\linewidth}{1pt}
\end{center}
\setcounter{page}{1}
\makeatletter


\beginsupplement
\section{Content}
This document completes the presentation of PA-GAN in the main paper with the following:
\begin{itemize}
	\item Theoretical proofs for Lemma 1 and Theorem 1 in Sec.~\ref{sec:sup_theory};
	\item Implementation details of PA-GAN in Sec.~\ref{sec:sup_implement};
	\item Additional ablation studies in Sec.~\ref{Ssec:experiments}; 
	\item Analysis of PA effectiveness as regularizer on the toy example in Sec.~\ref{sup_sec_toy_example};
	\item Exemplar synthetic images in Sec.~\ref{sec:syn samples};
	\item Results for the IS~\cite{Theis2016a} and KID~\cite{Binkowski2016MMDGAN} metrics in Sec.~\ref{kid-is};
	\item Network architectures and hyperparameter settings in Sec.~\ref{sec:networks}.
\end{itemize}
\section{Theoretical Framework of PA-GAN}\label{sec:sup_theory}

\subsection{Information Theory Viewpoint on the JS Divergence}\label{subsec:background}

Apart from quantifying distributions' similarity, the JS divergence has an information theory interpretation that inspires our approach. In accordance with the binary classification task of the discriminator, we introduce a binary random variable $s$ with a uniform distribution $\mathbb{P}_{\mathrm{s}}$. 
Associating $s=0$ and $s=1$ respectively with $\bvec{\vx}\sim \mathbb{P}_{\mathrm{d}}$ and $\bvec{\vx}\sim \mathbb{P}_{\mathrm{g}}$, we obtain a joint distribution function
\begin{align}
\mathbb{P}_{\mathrm{x},\mathrm{s}}(\bvec{\vx},s) \stackrel{\Delta}{=} \frac{\mathbb{P}_{\mathrm{d}}(\bvec{\vx})\delta[s]+\mathbb{P}_{\mathrm{g}}(\bvec{\vx})\delta[s-1]}{2},\label{def_Pxs}
\end{align}
where $\delta[\cdot]$ stands for the Kronecker delta function. The marginal distribution of $\mathbb{P}_{\mathrm{x},\mathrm{s}}$ with respect to $\bvec{\vx}$ (a.k.a. the mixture distribution) is equal to 
\begin{align}
\mathbb{P}_{\mathrm{m}}\stackrel{\Delta}{=}\mathbb{P}_{\mathrm{s}}(s=0)\mathbb{P}_{\mathrm{d}}+\mathbb{P}_{\mathrm{s}}(s=1)\mathbb{P}_{\mathrm{g}}= \frac{\mathbb{P}_{\mathrm{d}}+\mathbb{P}_{\mathrm{g}}}{2}.\label{Pm}
\end{align}
Computing the mutual information of the two random variables $s$ and $\bvec{\vx}$ based on $\mathbb{P}_{\mathrm{x},\mathrm{s}}$ is identical to computing the JS divergence between $\mathbb{P}_{\mathrm{d}}$ and $\mathbb{P}_{\mathrm{g}}$, i.e.,
\begin{align} 
I(\bvec{\vx}; s)  &= \frac{E_{\mathbb{P}_{\mathrm{m}}}\left[ p_{\mathrm{d}}(\bvec{\vx})\log p_{\mathrm{d}}(\bvec{\vx})\right] + E_{\mathbb{P}_{\mathrm{m}}}\left[ p_{\mathrm{g}}(\bvec{\vx})\log p_{\mathrm{g}}(\bvec{\vx})\right]}{2}=D_{\mathrm{JS}}\left(\mathbb{P}_{\mathrm{d}}\Vert\mathbb{P}_{\mathrm{g}} \right),\label{mi_js}
\end{align}

where $p_{\mathrm{d}}(\bvec{\vx})$ and $p_{\mathrm{g}}(\bvec{\vx})$ are density functions of $\mathbb{P}_{\mathrm{d}}$ and $\mathbb{P}_{\mathrm{g}}$ with respect to $\mathbb{P}_{\mathrm{m}}$.\footnote{Both $\mathbb{P}_{\mathrm{d}}$ and $\mathbb{P}_{\mathrm{g}}$ are absolutely continuous with respect to $\mathbb{P}_{\mathrm{m}}$. Therefore, their densities exist.} The minimum of the JS divergence $D_{\mathrm{JS}}\left(\mathbb{P}_{\mathrm{d}}\Vert\mathbb{P}_{\mathrm{g}} \right)$ equal to zero is attainable iff $\mathbb{P}_{\mathrm{d}}=\mathbb{P}_{\mathrm{g}}$, while zero mutual information indicates the independence between $\bvec{\vx}$ and $s$, yielding $\mathbb{P}_{\mathrm{x},\mathrm{s}}(\bvec{\vx},s)=\mathbb{P}_{\mathrm{m}}(\bvec{\vx})\mathbb{P}_{\mathrm{s}}(s)$. 

Exploiting the equality presented in (\ref{mi_js}), we proceed with proving Lemma 1, i.e., a series of JS divergence equalities. 

\subsection{Proof for Lemma 1}

\begin{lma}\label{lma1}
	Let $s\in\{0,1\}$ denote a random bit with uniform distribution $\mathbb{P}_{\mathrm{s}}(s)=\frac{\delta[s]+\delta[s-1]}{2}$, where $\delta[s]$ is the Kronecker delta. Associating $s$ with $\bvec{\vx}$, two joint distributions of $(\bvec{\vx}, s)$ are constructed as 
	\begin{align}
	\mathbb{P}_{\mathrm{x},\mathrm{s}}(\bvec{\vx},s) \stackrel{\Delta}{=} \frac{\mathbb{P}_{\mathrm{d}}(\bvec{\vx})\delta[s]+\mathbb{P}_{\mathrm{g}}(\bvec{\vx})\delta[s-1]}{2} ,\quad \mathbb{Q}_{\mathrm{x},\mathrm{s}}(\bvec{\vx},s) \stackrel{\Delta}{=} \frac{\mathbb{P}_{\mathrm{g}}(\bvec{\vx})\delta[s]+\mathbb{P}_{\mathrm{d}}(\bvec{\vx})\delta[s-1]}{2}.\label{PQ0}
	\end{align}
	Their JS divergence is equal to 
	\begin{align}
	D_{\mathrm{JS}}\left(\mathbb{P}_{\mathrm{x},\mathrm{s}}\Vert\mathbb{Q}_{\mathrm{x},\mathrm{s}} \right)=D_{\mathrm{JS}}\left(\mathbb{P}_{\mathrm{d}}\Vert\mathbb{P}_{\mathrm{g}} \right).\label{eqjoint}
	\end{align}
	Taking (\ref{PQ0}) as the starting point and with $\bvec{\vs}_{l}$ being a sequence of i.i.d. random bits of length $l$, the recursion of constructing the paired joint distributions of $(\bvec{\vx},\bvec{\vs}_{l})$ 
	\begin{align}
	\begin{array}{ll}
	\mathbb{P}_{\mathrm{x},\mathbf{s}_{l}}(\bvec{\vx},\bvec{\vs}_{l})  \stackrel{\Delta}{=} {\mathbb{P}_{\mathrm{x},\mathbf{s}_{l-1}}(\bvec{\vx},\bvec{\vs}_{l-1})\delta[s_l]/2+\mathbb{Q}_{\mathrm{x},\mathbf{s}_{l-1}}(\bvec{\vx},\bvec{\vs}_{l-1})\delta[s_l-1]/2}\\
	\mathbb{Q}_{\mathrm{x},\mathbf{s}_{l}}(\bvec{\vx},\bvec{\vs}_{l}) \stackrel{\Delta}{=} {\mathbb{Q}_{\mathrm{x},\mathbf{s}_{l-1}}(\bvec{\vx},\bvec{\vs}_{l-1})\delta[s_l]/2+\mathbb{P}_{\mathrm{x},\mathbf{s}_{l-1}}(\bvec{\vx},\bvec{\vs}_{l-1})\delta[s_l-1]/2}
	\end{array}\label{PQL}
	\end{align}
	results into a series of JS divergence equalities for $l=1,2,\dots,L$, i.e.,
	\begin{align}
	D_{\mathrm{JS}}\left(\mathbb{P}_{\mathrm{d}}\Vert\mathbb{P}_{\mathrm{g}} \right)
	=D_{\mathrm{JS}}\left(\mathbb{P}_{\mathrm{x},\mathbf{s}_1}\Vert \mathbb{Q}_{\mathrm{x},\mathbf{s}_1} \right)=\cdots=
	D_{\mathrm{JS}}\left(\mathbb{P}_{\mathrm{x},\mathbf{s}_L}\Vert \mathbb{Q}_{\mathrm{x},\mathbf{s}_L} \right)
	\label{eqjoint2}.
	\end{align}
\end{lma}

\begin{proof}
Starting from the single bit $s$, the two joint distributions $\mathbb{P}_{\mathrm{x},\mathrm{s}}$ and $\mathbb{Q}_{\mathrm{x},\mathrm{s}}$ differ from each other by their opposite way of associating the bit $s \in \left\{0,1\right\}$ with the data and synthetic samples. %
Their marginals with respect to $\bvec{\vx}$ are identical and equal the mixture distribution $\mathbb{P}_{\mathrm{m}}$, being neither the data nor the model distribution, in contrast to the framework of \cite{ali2017}.
	
The joint distribution $\mathbb{P}_{\mathrm{x},\mathrm{s}}$ has yielded the mutual information $I(\bvec{\vx}; s)$ with the equality in (\ref{mi_js}). By analogy, we compute the mutual information $\tilde{I}(\bvec{\vx}; s)$ between $\bvec{\vx}$ and $s$ which follow $\mathbb{Q}_{\mathrm{x},\mathrm{s}}$ with the equality:
\begin{align}
\tilde{I}(\bvec{\vx}; s)=D_{\mathrm{JS}}\left(\mathbb{P}_{\mathrm{d}}\Vert\mathbb{P}_{\mathrm{g}} \right).\label{app:Iq}
\end{align}

The combination of (\ref{mi_js}) and (\ref{app:Iq}) leads to
\begin{align}
D_{\mathrm{JS}}\left(\mathbb{P}_{\mathrm{d}}\Vert\mathbb{P}_{\mathrm{g}} \right) = \frac{I(\bvec{\vx}; s)+\tilde{I}(\bvec{\vx}; s)}{2}.
\end{align}
Rewriting mutual information as KL divergence yields:
\begin{align}
D_{\mathrm{JS}}\left(\mathbb{P}_{\mathrm{d}}\Vert\mathbb{P}_{\mathrm{g}} \right) = \frac{D_{\mathrm{KL}}\left(\mathbb{P}_{\mathrm{x},\mathrm{s}}\Vert\mathbb{P}_{\mathrm{m}}\mathbb{P}_{\mathrm{s}} \right)+D_{\mathrm{KL}}\left(\mathbb{Q}_{\mathrm{x},\mathrm{s}}\Vert\mathbb{P}_{\mathrm{m}}\mathbb{P}_{\mathrm{s}}\right)}{2},\label{js_kl}
\end{align}
where $\mathbb{P}_{\mathrm{m}}$ and $\mathbb{P}_{\mathrm{s}}$ are the common marginals of $\mathbb{P}_{\mathrm{x},\mathrm{s}}$ and $\mathbb{Q}_{\mathrm{x},\mathrm{s}}$ with respect to $\bvec{\vx}$ and $s$. By further identifying
\begin{align}
\mathbb{P}_{\mathrm{m}}(\bvec{\vx}) \mathbb{P}_{\mathrm{s}}(s) = \frac{\mathbb{P}_{\mathrm{x},\mathrm{s}}(\bvec{\vx},s)+\mathbb{Q}_{\mathrm{x},\mathrm{s}}(\bvec{\vx},s)}{2}
\end{align}
and plugging it into (\ref{js_kl}), 
we finally reach to
\begin{align}
D_{\mathrm{JS}}\left(\mathbb{P}_{\mathrm{d}}\Vert\mathbb{P}_{\mathrm{g}} \right)&=D_{\mathrm{JS}}\left(\mathbb{P}_{\mathrm{x},\mathrm{s}}\Vert\mathbb{Q}_{\mathrm{x},\mathrm{s}} \right)\label{eqjoint2}
\end{align}
by the definition of JS divergence. 

It is worth noting that the equivalence holds even if the feasible solution set of $\mathbb{P}_{\mathrm{g}}$ determined by $G$ does not include the data distribution $\mathbb{P}_{\mathrm{d}}$. This is of practical interest as it is often difficult to guarantee the fulfillment of such premise when modeling $G$ by means of neural networks. 

Replacing the data and model distributions $\mathbb{P}_{\mathrm{d}}$ and $\mathbb{P}_{\mathrm{g}}$ respectively with $\mathbb{P}_{\mathrm{x},\mathrm{s}}$ and $\mathbb{Q}_{\mathrm{x},\mathrm{s}}$, we can systematically add a new bit with the same derivation as above. Repeating this procedure $L$ times eventually yields the recursively constructed $\{\mathbb{P}_{\mathrm{x},\mathbf{s}_{l}},\mathbb{Q}_{\mathrm{x},\mathbf{s}_{l}}\}_{l=1,\dots,L}$ followed by a sequence of JS divergence equalities
\begin{align}
D_{\mathrm{JS}}\left(\mathbb{P}_{\mathrm{d}}\Vert\mathbb{P}_{\mathrm{g}} \right)=\cdots=D_{\mathrm{JS}}\left(\mathbb{P}_{\mathrm{x},\mathbf{s}_{l-1}}\Vert \mathbb{Q}_{\mathrm{x},\mathbf{s}_{l-1}} \right)=\cdots=D_{\mathrm{JS}}\left(\mathbb{P}_{\mathrm{x},\mathbf{s}_L}\Vert \mathbb{Q}_{\mathrm{x},\mathbf{s}_L} \right).\label{eqjoint3}
\end{align}
\end{proof}

\subsection{Proof for Theorem 1}
\begin{thm}\label{thm}
	The min-max optimization problem of GANs~\cite{goodfellow2014generative} is equivalent to
	\begin{align}
	\min_G \max_{D} 
	\mathbb{E}_{\mathbb{P}_{\mathrm{x},\mathbf{s}_{l}}}\left\{\log \left[D(\bvec{\vx},\bvec{\vs}_l)\right]\right\}   + \mathbb{E}_{\mathbb{Q}_{\mathrm{x},\mathbf{s}_{l}}}\left\{\log\left[1-D(\bvec{\vx},\bvec{\vs}_l)\right]\right\} \quad\forall l\in\{1,2,\dots,L\},\label{min-max2}
	\end{align}
	where the two joint distributions, i.e., $\mathbb{P}_{\mathrm{x},\mathbf{s}_{l}}$ and $\mathbb{Q}_{\mathrm{x},\mathbf{s}_{l}}$, are defined in (\ref{PQL}) and the function $D$ maps $(\bvec{\vx},\bvec{\vs}_l)\in\mathcal{X}\times \{0,1\}^l$ onto $[0,1]$. For a fixed $G$, the optimal $D$ is
	\begin{align}
	D^*(\bvec{\vx},\bvec{\vs}_l)=\frac{\mathbb{P}_{\mathrm{x},\mathbf{s}_{l}}(\bvec{\vx},\bvec{\vs}_l)}{\mathbb{P}_{\mathrm{x},\mathbf{s}_{l}}(\bvec{\vx},\bvec{\vs}_l)+\mathbb{Q}_{\mathrm{x},\mathbf{s}_{l}}(\bvec{\vx},\bvec{\vs}_l)}=\frac{\mathbb{P}_{\mathrm{d}}(\bvec{\vx})}{\mathbb{P}_{\mathrm{d}}(\bvec{\vx})+\mathbb{Q}_{\mathrm{d}}(\bvec{\vx})},
	\end{align}
	whereas the attained inner maximum equals $D_{\mathrm{JS}}\left(\mathbb{P}_{\mathrm{x},\mathbf{s}_l}\Vert \mathbb{Q}_{\mathrm{x},\mathbf{s}_l}\right)= D_{\mathrm{JS}}\left(\mathbb{P}_{\mathrm{d}}\Vert\mathbb{P}_{\mathrm{g}} \right)$ for $l=1,2,\dots,L$.
\end{thm}

\begin{proof}
Analogous to the proofs for GANs~\cite[Sec.4]{goodfellow2014generative}, we can construct a binary classification task for computing JS divergences, i.e.,
\begin{align}
D_{\mathrm{JS}}\left(\mathbb{P}_{\mathrm{x},\mathbf{s}_{l}}\Vert \mathbb{Q}_{\mathrm{x},\mathbf{s}_{l}} \right)=\max_D \mathbb{E}_{\mathbb{P}_{\mathrm{x},\mathbf{s}_l}}\left\{\log \left[D(\bvec{\vx},\bvec{\vs}_l)\right]\right\}+\mathbb{E}_{ \mathbb{Q}_{\mathrm{x},\mathbf{s}_l}}\left\{\log \left[1-D(\bvec{\vx},\bvec{\vs}_l)\right]\right\}\quad \forall l,
\end{align}
where the optimal $D^*$ equals
\begin{align}
D^*(\bvec{\vx},\bvec{\vs}_l)=\frac{\mathbb{P}_{\mathrm{x},\mathbf{s}_{l}}(\bvec{\vx},\bvec{\vs}_l)}{\mathbb{P}_{\mathrm{x},\mathbf{s}_{l}}(\bvec{\vx},\bvec{\vs}_l)+\mathbb{Q}_{\mathrm{x},\mathbf{s}_{l}}(\bvec{\vx},\bvec{\vs}_l)}\stackrel{(a)}{=}\frac{\mathbb{P}_{\mathrm{d}}(\bvec{\vx})}{\mathbb{P}_{\mathrm{d}}(\bvec{\vx})+\mathbb{Q}_{\mathrm{d}}(\bvec{\vx})}.
\end{align}
The equality $(a)$ in above is based on the recursive construction of $\mathbb{P}_{\mathrm{x},\mathbf{s}_{l}}$ and $\mathbb{Q}_{\mathrm{x},\mathbf{s}_{l}}$ from $\mathbb{P}_{\mathrm{d}}$ and $\mathbb{P}_{\mathrm{g}}$.

The equalities in (\ref{eqjoint}) imply that for any given pair $(\mathbb{P}_{\mathrm{d}},\mathbb{P}_{\mathrm{g}})$ the correspondingly constructed joint distribution pair $(\mathbb{P}_{\mathrm{x},\mathbf{s}_{l}},\mathbb{Q}_{\mathrm{x},\mathbf{s}_{l}})$ yields the same JS divergence. For this reason, we can use the two JS divergences interchangeably as the objective function while optimizing $\mathbb{P}_{\mathrm{g}}$, yielding
\begin{align}
&\min_G \max_D  \mathbb{E}_{\mathbb{P}_{\mathrm{d}}}\left\{\log \left[D(\bvec{\vx})\right]\right\}+\mathbb{E}_{ \mathbb{P}_{\mathrm{g}}}\left\{\log \left[1-D(\bvec{\vx})\right]\right\}\notag\\
&\quad\quad\quad\quad\equiv \min_G \max_D \mathbb{E}_{\mathbb{P}_{\mathrm{x},\mathbf{s}_l}}\left\{\log \left[D(\bvec{\vx},\bvec{\vs}_l)\right]\right\}+\mathbb{E}_{ \mathbb{Q}_{\mathrm{x},\mathbf{s}_l}}\left\{\log \left[1-D(\bvec{\vx},\bvec{\vs}_l)\right]\right\}\quad \forall l.
\end{align}
\end{proof}

\subsection{Generalization of Lemma 1}\label{Ssub:lemma1_gen}
In this work, we base the development of PA on Lemma 1 and Theorem 1. From a broader perspective, the random bits $\bvec{\vs}$ can be any generic random variables applicable for generative modelling.
\begin{pro}
Let $\bvec{\vs}$ denote a random variable with two unequal distributions $\mathbb{P}_{\mathrm{s},a}$ and $\mathbb{P}_{\mathrm{s},b}$. Together with the two distributions $\mathbb{P}_{\mathrm{d}}$ and $\mathbb{P}_{\mathrm{g}}$ of $\bvec{\vx}$, two joint distributions are constructed as follows:
\begin{align}
\begin{array}{ll}
\mathbb{P}_{\mathrm{x},\mathrm{s}}(\bvec{\vx},\bvec{\vs}) =\frac{\mathbb{P}_{\mathrm{d}}(\bvec{\vx}) \mathbb{P}_{\mathrm{s},a}(\bvec{\vs})+\mathbb{P}_{\mathrm{g}}(\bvec{\vx}) \mathbb{P}_{\mathrm{s},b}(\bvec{\vs})}{2} \\
\mathbb{Q}_{\mathrm{x},\mathrm{s}}(\bvec{\vx},\bvec{\vs}) =\frac{\mathbb{P}_{\mathrm{d}}(\bvec{\vx}) \mathbb{P}_{\mathrm{s},b}(\bvec{\vs})+\mathbb{P}_{\mathrm{g}}(\bvec{\vx}) \mathbb{P}_{\mathrm{s},a}(\bvec{\vs})}{2}
\end{array}.\label{propPQ}
\end{align}
The mutual information $I(\bvec{\vx}; \bvec{\vs})$ and $\tilde{I}(\bvec{\vx}; \bvec{\vs})$, with respect to $\mathbb{P}_{\mathrm{x},\mathrm{s}}$ and $\mathbb{Q}_{\mathrm{x},\mathrm{s}}$, are minimized to zero if $\mathbb{P}_{\mathrm{d}}=\mathbb{P}_{\mathrm{g}}$. When $\mathbb{P}_{\mathrm{s},a}$ and $\mathbb{P}_{\mathrm{s},b}$ have non-overlapped supports, the JS divergence between $\mathbb{P}_{\mathrm{x},\mathrm{s}}$ and $\mathbb{Q}_{\mathrm{x},\mathrm{s}}$ equals the JS divergence between $\mathbb{P}_{\mathrm{d}}$ and $\mathbb{P}_{\mathrm{g}}$, i.e., $D_{\mathrm{JS}}\left(\mathbb{P}_{\mathrm{d}}\Vert\mathbb{P}_{\mathrm{g}} \right)=D_{\mathrm{JS}}\left(\mathbb{P}_{\mathrm{x},\mathrm{s}}\Vert\mathbb{Q}_{\mathrm{x},\mathrm{s}} \right)$.
\end{pro}

\begin{proof}
The mutual information between $\bvec{\vx}$ and $\bvec{\vs}$ is minimized and equal zero if they are independent. Under the condition $\mathbb{P}_{\mathrm{s},a}\not=\mathbb{P}_{\mathrm{s},b}$, the two joint distributions $\mathbb{P}_{\mathrm{x},\mathrm{s}}$ and $\mathbb{Q}_{\mathrm{x},\mathrm{s}}$ become factorizable if $\mathbb{P}_{\mathrm{d}}=\mathbb{P}_{\mathrm{g}}$. Analogous to the proof of Lemma 1, the JS divergence between $\mathbb{P}_{\mathrm{x},\mathrm{s}}$ and $\mathbb{Q}_{\mathrm{x},\mathrm{s}}$ equals the mean of $I(\bvec{\vx}; \bvec{\vs})$ and $\tilde{I}(\bvec{\vx}; \bvec{\vs})$, i.e., 
\begin{align}
D_{\mathrm{JS}}\left(\mathbb{P}_{\mathrm{x},\mathrm{s}}\Vert\mathbb{Q}_{\mathrm{x},\mathrm{s}} \right) = \frac{I(\bvec{\vx}; \bvec{\vs})+\tilde{I}(\bvec{\vx}; \bvec{\vs})}{2}.
\end{align}
Expressing mutual information as KL divergence plus the condition that $\mathbb{P}_{\mathrm{s},a}$ and $\mathbb{P}_{\mathrm{s},b}$ have non-overlapped supports, we reach to (\ref{mi_js}) for both $I(\bvec{\vx}; \bvec{\vs})$ and $\tilde{I}(\bvec{\vx}; \bvec{\vs})$ and thereby conclude the proof.
\end{proof}
Lemma 1 is a special case of Proposition 1, namely, $\mathbb{P}_{\mathrm{s},a}(s)=\delta[s]$ and $\mathbb{P}_{\mathrm{s},b}(s)=\delta[s-1]$.
\section{Implementation Details of PA-GAN}\label{sec:sup_implement}

\subsection{Input and Feature Space Augmentation}

As being presented in Sec. 3.2, we spatially replicate each augmentation bit and perform depth concatenation with 
the input $\bvec{\vx}$ or its learned feature maps at the intermediate hidden layers. %
After concatenation along the channel axis, the input layer or the hidden layer then process such augmented input. For instance, in the case of a convolutional layer, it processes the augmented input as
\begin{align}
\mathrm{conv}(\phi(\bvec{\vx}),s_1,\dots,s_l)=\mathrm{conv}(\phi(\bvec{\vx}))+\sum_l\mathrm{conv}(s_l)\label{aug_conv}
\end{align}
where the kernel width/height, stride and padding type used for filtering the augmentation bits are the same as that of $\phi(\bvec{\vx})$.\footnote{\url{https://github.com/boschresearch/PA-GAN/blob/master/pagan_ops.py}} Depending on the augmentation space, here $\phi(\bvec{\vx})$ collectively denotes either the input $\bvec{\vx}$ or its feature maps. When spectral normalization is in use, the power method is applied to estimate the largest singular value of the filter matrix that processes the augmented input. In case of augmenting the input to a residual block, the augmentation bits are passed along with $\bvec{\vx}$ or its feature maps into the first convolutional layer in the main branch as well as into the shortcut connection. We bypass the shortcut connection if it is an identity mapping.

When progression scheduling increases the augmentation level, a new set of filter coefficients are instantiated to process the new augmentation bit according to (\ref{aug_conv}). They are initialized by random Gaussian variables with the mean and variance computed from the existing filter coefficients for $\phi(\bvec{\vx})$. Before filtering, each augmentation bit can be additionally modulated by two trainable parameters $\{\lambda_l,\beta_l\}$. The scaling parameter $\lambda_l$ is initialized with the mean value of the previous ones, where the first one, i.e., $\lambda_1$, is initialized as one. The offset parameters $\{\beta_l\}$ are always initialized as zeros.

\subsection{Mini-batch Discrimination} \label{subsec:minibatch}
Each mini-batch is constructed with the same number of real data samples, synthetic samples and bit sequences. %
Each bit sequence is randomly sampled and associated with one real and one synthetic sample. %
Based on the checksums of the formed pairs, we can decide their correct class and feed it into the discriminator to compute the cross-entropy loss. This way of generating $(\bvec{\vx},\bvec{\vs})$ guarantees a balanced number of TRUE/FAKE samples, forming the two mini-batches $\mathcal{B}_{\mathrm{tr}}$ and $\mathcal{B}_{\mathrm{fk}}$.

\subsection{Warm-up Phase of Progression}\label{sec:warm-up}

At the beginning of the new augmentation level the discriminator is ignorant about this disruptive change and as the bit $s=1$ flips the reference label it will lead to about $50\%$ discriminator errors in one mini-batch. Aiming at a smooth transition from the current augmentation level to the new one, here we introduce two warm-up mechanisms that are usable when the discriminator exhibits deficiency in timely coping with the new augmentation level.

The first mechanism instantiates an Adam optimizer, independent of the ones for $D$ and $G$, to solely train the newly introduced weights right after progressing to the new level. It takes the $D$ loss and can use the same learning hyperparameters as those of the $D$ optimizer. After multiple iterations (e.g., $1$\unit{k}), we continue with the original alternation between the $D$ and $G$ optimizer, where the new weights together with the existing ones of the discriminator network are handled by the $D$ optimizer.

According to Lemma 1, the augmentation bits shall follow a uniform distribution, i.e., $\mathbb{P}(s=1)=p$ and $\mathbb{P}(s=0)=1-p$ with $p=0.5$. As the new augmentation bit taking on the value one causes discriminator errors, the second mechanism temporally adopts a non-uniform distribution when kicking off a new augmentation level. %
Namely, we can on purpose create more $0$s than $1$s by linearly increasing $p$ from $0$ and $0.5$ within a given number of iterations, e.g., $5$\unit{k}. 

\subsection{Loss Functions}
In this work, we experimented of using PA with the following loss functions of GANs.
\paragraph{Non-saturating (NS) loss.} 
The cross-entropy loss for $D$ is given as
\begin{align}
\min_D -\mathbb{E}_{\mathbb{P}_{\mathrm{x},\mathbf{s}_{l}}}\left\{\log D(\bvec{\vx},\bvec{\vs}_l)\right\}-\mathbb{E}_{\mathbb{Q}_{\mathrm{x},\mathbf{s}_{l}}}\left\{\log \left[1-D(\bvec{\vx},\bvec{\vs}_l)\right]\right\}.
\end{align}
Since both distribution $\mathbb{P}_{\mathrm{x},\mathbf{s}_{l}}$ and $\mathbb{Q}_{\mathrm{x},\mathbf{s}_{l}}$ involve synthetic samples, the non-saturating (NS) loss for $G$~\cite{goodfellow2014generative} is reformulated as
\begin{align}
\min_G -\mathbb{E}_{\mathbb{Q}_{\mathrm{x},\mathbf{s}_{l}}}\left\{\log D(\bvec{\vx},\bvec{\vs}_l)\right\}-\mathbb{E}_{\mathbb{P}_{\mathrm{x},\mathbf{s}_{l}}}\left\{\log \left[1-D(\bvec{\vx},\bvec{\vs}_l)\right]\right\}.\label{ns-loss}
\end{align}
During training, the two expectations are approximated by averaging over the samples in the TRUE/FAKE mini-batches $\mathcal{B}_{\mathrm{tr}}$ and $\mathcal{B}_{\mathrm{fk}}$, which construction is discussed in Sec.~\ref{subsec:minibatch}.

\paragraph{Hinge loss.} 
Instead of cross-entropy loss, $D$ can also be trained using the hinge loss
\begin{align}
	\min_D \mathbb{E}_{\mathbb{P}_{\mathrm{x},\mathbf{s}_{l}}}\left\{\max \left[0, 1-D(\bvec{\vx},\bvec{\vs}_l)\right]\right\}+\mathbb{E}_{\mathbb{Q}_{\mathrm{x},\mathbf{s}_{l}}}\left\{\max\left[0, 1+ D(\bvec{\vx},\bvec{\vs}_l)\right]\right\}.
\end{align}
Accordingly, the $G$ loss is adapted to
\begin{align}
\min_G \mathbb{E}_{\mathbb{P}_{\mathrm{x},\mathbf{s}_{l}}}\left\{D(\bvec{\vx},\bvec{\vs}_l)\right\}-\mathbb{E}_{\mathbb{Q}_{\mathrm{x},\mathbf{s}_{l}}}\left\{ D(\bvec{\vx},\bvec{\vs}_l)\right\}.
\end{align}

\paragraph{WGAN-GP.}\label{subsec:wgan}
In the main paper, we have focused on generative modeling with JS divergence. It is also possible to interchange the JS divergence with the Wasserstein distance and then cast GAN training into WGAN-GP training~\cite{Arjovsky2017WGAN}. Wasserstein distance is weaker than JS divergence and $D$ termed critic in WGAN no longer solves the classification task. So, we alternatively exploit the stochastic model averaging role of the augmentation bits rather than their regularization role.

Briefly, with the Kantorovich-Rubinstein duality, minimizing the Wasserstein distance between $\mathbb{P}_{\mathrm{d}}$ and $\mathbb{P}_{\mathrm{g}}$ is transformed into the following two-player game
\begin{align}
\min_G \max_D \mathbb{E}_{\mathbb{P}_{\mathrm{x},\mathbf{s}_{l}}}\left\{D(\bvec{\vx},\bvec{\vs}_l)\right\}-\mathbb{E}_{\mathbb{Q}_{\mathrm{x},\mathbf{s}_{l}}}\left\{ D(\bvec{\vx},\bvec{\vs}_l)\right\}.
\end{align}
Ideally, $D$ in the context of WGAN should be $1$-Lipschitz continuous. As a pragmatic relaxation on this constraint, a gradient penalty (GP)~\cite{gulrajani_NIPS2017} is commonly added to the objective function when optimizing $D$.

Within the same mini-batch of $\bvec{\vx}\sim \mathbb{P}_{\mathrm{d}}$ and $\bvec{\vx}\sim \mathbb{P}_{\mathrm{g}}$, we draw $M$ mini-batches $\bvec{\vs}\sim \mathbb{P}_{\mathrm{s}}$ of the same size. Combining each of them with the data and synthetic samples, we create $M$ mini-batches for approximating the expectations in the objective function
\begin{align}
&\mathbb{E}_{\mathbb{P}_{\mathrm{x},\mathbf{s}_{l}}}\left\{D(\bvec{\vx},\bvec{\vs}_l)\right\}-\mathbb{E}_{\mathbb{Q}_{\mathrm{x},\mathbf{s}_{l}}}\left\{ D(\bvec{\vx},\bvec{\vs}_l)\right\}\notag\\
&\quad  \approx L_{m}\stackrel{\Delta}{=}\frac{1}{\vert \mathcal{B}_{\mathrm{tr},m}\vert }\sum_{(\bvec{\vx},\bvec{\vs}_l)\in\mathcal{B}_{\mathrm{tr},m}} \hspace{-0.3cm} D(\bvec{\vx},\bvec{\vs}_l) -\frac{1}{\vert \mathcal{B}_{\mathrm{fk},m}\vert }\sum_{(\bvec{\vx},\bvec{\vs}_l)\in\mathcal{B}_{\mathrm{fk},m}} \hspace{-0.3cm} D(\bvec{\vx},\bvec{\vs}_l),\quad m = 1,\dots,M.
\end{align}
The critic $D$ of WGAN-GP is trained to maximize the averaged loss $L_{m}$ across the $M$ mini-batches, making use of stochastic model averaging. The generator $G$ is then trained to minimize the maximum of $\{L_m\}$, $m = 1,\dots,M$, i.e. picking the best performing case of the critic, as a good quality of the critic $D$ is important to the optimization process of $G$ in the context of WGAN. With single bit augmentation of $\mathtt{PA\text{ }(feat)}$ and two draws per minibatch, we can improve WGAN-GP of $\mathtt{SN\text{ }DCGAN}$ on CIFAR10 from $25.0$ to $23.9$ FID. 
Here, we boost the diversity of the two draws by choosing them with opposite checksums.

\section{Additional Ablation Studies}\label{Ssec:experiments}
In this section, we provide additional ablation studies of PA. Complementary to Table~1 in Sec. 4.1., an ablation study on the choice of augmentation space is conducted in Sec.~\ref{Ssubsec:abl aug level}, evaluating PA across input, low- and high-level feature space augmentation. One important finding in Sec.~4.2. of the main paper is that dropout and PA are complementary and mutually beneficial. In Sec.~\ref{Ssubsec:dropout}, we report our detailed investigation on the dropout regularization followed by evaluation of its combination with PA across the datasets and architectures. %
The two time-scale update rule (TTUR)~\cite{heuselttur2017}, updating the discriminator and generator with different learning rates, is notoriously helpful to stabilize GAN training. In Sec.~\ref{Ssubsec:ttur}, we examine the performance of PA under different TTURs and then compare it with the adaptive learning rate. %

\subsection{Ablation Study on Augmentation Space}\label{Ssubsec:abl aug level}

\begin{table*}[t!]
	\setlength{\tabcolsep}{0.5em} 
	\renewcommand{\arraystretch}{1.1}
	\centering
	\caption{Median FIDs of input and feature space augmentation across five random runs. We experiment with augmenting input and features at different intermediate layers, e.g. $\mathtt{feat_{N/4}}$ denotes layer with the spatial dimension $N/4$, where $N$ is the input image dimension.} \label{table_pa_auglevel}	
	\vspace{0.5em}
	\begin{tabular}{ll|ccccc} 
		\rowcolor{verylightgray}
		& & \multicolumn{5}{c}{\footnotesize{}{\text{$\mathtt{PA}$}}} \tabularnewline 
		
		\rowcolor{verylightgray}
		
		\multirow{-2}{*}{\footnotesize{}{\text{$\mathtt{Method}$}}} &	\multirow{-2}{*}{\footnotesize{}{\text{$\mathtt{Dataset}$}}}	& \footnotesize{}{\text{\xmark}} &\footnotesize{}{\text{$\mathtt{input\text{ }(N)}$}}  & \footnotesize{}{\text{$\mathtt{feat_{N/2}}$}} & \footnotesize{}{\text{$\mathtt{feat_{N/4}}$}} & \footnotesize{}{\text{$\mathtt{feat_{N/8}}$}}\tabularnewline 
		
	 &  \multirow{1}{*}{\text{$\mathtt{CIFAR10}$}}	& \text{\footnotesize$26.0$} 
		&\text{\footnotesize$\mathbf{22.2}$} & \text{\footnotesize$22.8$} &\textbf{\footnotesize$22.7$} & \textbf{\footnotesize$22.6$} \tabularnewline 
		
		\multirow{-2}{*}{\text{$\mathtt{SN\text{ }DCGAN\text{ - }NS\text{ }Loss}$}} &  \multirow{1}{*}{\text{$\mathtt{CELEBA\text{-}HQ}$}}&\text{\footnotesize$24.3$} 
		&\text{\footnotesize$20.8$} &\text{\footnotesize$19.6$} & \footnotesize{\text{$\mathbf{18.8}$}} & \textbf{\footnotesize$\mathbf{18.8}$} \tabularnewline
				\arrayrulecolor{verylightgray}	\hline
				
	 &  	\multirow{1}{*}{\text{$\mathtt{CIFAR10}$}}	& \text{\footnotesize$18.8$} &  \text{\footnotesize$\mathbf{16.1}$} &  \text{\footnotesize$16.3$}& \textbf{\footnotesize${16.3}$} & \text{-}\tabularnewline

\multirow{-2}{*}{\text{$\mathtt{SA\text{ }GAN\text{ }(sBN)\text{ - }Hinge\text{ }Loss}$}}	 &  		\multirow{1}{*}{\text{$\mathtt{CELEBA\text{-}HQ}$}}	& \text{\footnotesize$17.8$} &  \text{\footnotesize$\mathbf{15.4}$} &  \text{\footnotesize$\mathbf{15.4}$}& \textbf{\footnotesize${16.4}$} & \textbf{\footnotesize${15.8}$}  \tabularnewline 

\arrayrulecolor{verylightgray}	\hline

	\end{tabular}	
\end{table*}

In the main paper, in Table 1 of Sec. 4.1. we reported the FID scores achieved by PA,
by augmenting either the input - $\mathtt{PA}\text{ }(\mathtt{input})$, or its features with spatial dimension $N/8$ - $\mathtt{PA}\text{ }(\mathtt{feat_{N/8}})$, where $N$ is the input image dimension (see Sec.~\ref{sec:networks} for the detailed configuration). 
Here, we further perform the ablation study on the choice of the augmentation space across two datasets (CIFAR10 and CELEBA-HQ) and two architectures (SN DCGAN and SA GAN). From Table~\ref{table_pa_auglevel}, we observe the stable performance improvement across all configurations, inline with Table 1 of the main paper. The performance difference across different feature space augmentations is generally small (less than one FID point).

\begin{table}[t!]
	\vspace{-1em}
	\setlength{\tabcolsep}{0.22em} 
	\renewcommand{\arraystretch}{1.1}
	\centering
	\caption{Median FIDs (across five random runs) of $\mathtt{Dropout}$ and $\mathtt{SpatialDropout}$ applied on the input layer or intermediate layers with different keep rates on CIFAR10 using SN DCGAN.} \label{table:dropout}\vspace{0.5em}
\begin{tabular}{c|cccc|cccc} 
	\rowcolor{verylightgray}
	&    \multicolumn{4}{c}{\footnotesize{}{\text{$\mathtt{Dropout}$}}} &  \multicolumn{4}{c}{\footnotesize{}{\text{$\mathtt{SpatialDropout}$}}}  \tabularnewline 
		\rowcolor{verylightgray}
	\multirow{-2}{*}{\footnotesize{}{\text{$\mathtt{Keep\text{ }rate}$}} } &\footnotesize{}{\text{$\mathtt{input\text{ }(N)}$}} & \footnotesize{}{\text{$\mathtt{feat_{N/2}}$}} & \footnotesize{}{\text{$\mathtt{feat_{N/4}}$}} & \footnotesize{}{\text{$\mathtt{feat_{N/8}}$}}&\footnotesize{}{\text{$\mathtt{input\text{ }(N)}$}} & \footnotesize{}{\text{$\mathtt{feat_{N/2}}$}} & \footnotesize{}{\text{$\mathtt{feat_{N/4}}$}} & \footnotesize{}{\text{$\mathtt{feat_{N/8}}$}} \tabularnewline 
	\footnotesize{}{\text{$1.0$}} & \multicolumn{8}{c}{\text{\footnotesize$26.0$}} \tabularnewline
	\arrayrulecolor{verylightgray}	\hline 
	\footnotesize{}{\text{$0.95$}} & \text{\footnotesize$25.5$} &  \text{\footnotesize$25.6$} &  \text{\footnotesize$24.1$} &  \text{\footnotesize$25.3$} & 
									 \text{\footnotesize$26.0$} &  \text{\footnotesize$25.3$} &  \text{\footnotesize$24.9$} &  \text{\footnotesize$26.0$} \tabularnewline 
	\footnotesize{}{\text{$0.9$}}  & \text{\footnotesize$26.4$} &  \text{\footnotesize$25.1$} &  \text{\footnotesize$23.4$} &  \text{\footnotesize$24.6$} & 
	                                 \text{\footnotesize$26.2$} &  \text{\footnotesize$25.3$} &  \text{\footnotesize$24.0$} &  \text{\footnotesize$25.8$} \tabularnewline 
	\footnotesize{}{\text{$0.7$}} & \text{\footnotesize$28.0$} &  \text{\footnotesize$25.6$} &  \text{\footnotesize$\mathbf{22.1}$} &  \text{\footnotesize$24.4$} &
	                                \text{\footnotesize$27.6$} &  \text{\footnotesize$26.1$} &  \text{\footnotesize$\underline{23.4}$} &  \text{\footnotesize$25.3$} \tabularnewline 
	\footnotesize{}{\text{$0.5$}} & \text{\footnotesize$27.1$} &  \text{\footnotesize$25.9$} &  \text{\footnotesize$23.1$} &  \text{\footnotesize$24.0$} & 
									\text{\footnotesize$29.7$} &  \text{\footnotesize$26.9$} &  \text{\footnotesize$24.1$} &  \text{\footnotesize$25.4$}\tabularnewline 
	\footnotesize{}{\text{$0.3$}} & \text{\footnotesize$27.7$} &  \text{\footnotesize$25.6$} &  \text{\footnotesize$22.4$} &  \text{\footnotesize$24.6$} & 
	                                \text{\footnotesize$31.3$} &  \text{\footnotesize$28.8$} &  \text{\footnotesize$24.6$} &  \text{\footnotesize$25.8$} \tabularnewline 
	\footnotesize{}{\text{$0.1$}} & \text{\footnotesize$32.3$} &  \text{\footnotesize$28.6$} &  \text{\footnotesize$24.3$} &  \text{\footnotesize$23.9$} & 
	                                 \text{\footnotesize$45.7$} &  \text{\footnotesize$37.7$} &  \text{\footnotesize$28.8$} &  \text{\footnotesize$25.8$}  \tabularnewline 
\end{tabular}
\end{table}

\subsection{Ablation Study on Dropout and its Combination with PA}\label{Ssubsec:dropout}
In Sec.~4.2. of the main paper, we have shown the effectiveness of using dropout, particularly, in combination with the proposed PA. In this part we report further ablations for both techniques.

We start from applying dropout at the input layer and different intermediate layers. Note that, in contrast to dropout, we apply PA directly on the input and not on the input layer. %
In addition, we experiment with different keep rates of the dropout, i.e. $\{0.1,0.3,0,5,0.7,0.9,0.95\}$. Table~\ref{table:dropout} reports the FID scores achieved with different dropout configurations. In contrast to PA (see Table~\ref{table_pa_auglevel} or Table 1 in the main paper), the performance of dropout is very dependent on the applied layer and the selected keep rate. The feature space with the spatial dimension $N/4$ together with the keep rate $0.7$ is the best performing setting on CIFAR10 with SN DCGAN.

We further note that the binary dropout mask is independently drawn for each entry of the input or intermediate layer outputs (each convolution feature map activation is
"dropped-out" independently). In addition, we also experiment with the \emph{spatial} dropout ($\mathtt{SpatialDropout}$) \cite{Tompson2015EfficientOL}, which randomly drops the entire feature maps instead of individual elements. The results in Tables~\ref{table:dropout} show that the entry-wise dropout outperforms the spatial dropout in the context of GAN training, i.e., FID $22.1$ vs. $23.4$. Therefore we only consider the entry-wise dropout for comparison with PA in the main paper.

In Table 3 of the main paper, we have successfully combined dropout at its best setting with PA on CIFAR10 with SN DCGAN and SA GAN. Table~\ref{table:pa_dropout} and~\ref{tab:dropoutpa} additionally report the FID improvements where dropout is applied at different intermediate layers and keep rates. In all configurations, PA provides complementary gains. %
Note that, for CELEBA-HQ $\mathtt{Dropout}$ alone in Table~\ref{table:pa_dropout} only has a marginal performance improvement over the baseline, whereas its combination with $\mathtt{PA}$ leads to larger performance boost. Overall, Table~\ref{table:pa_dropout},~\ref{tab:dropoutpa} plus Table~3 in the main paper confirms the effectiveness of exploiting both techniques. Adding PA is beneficial independent of the dropout settings (keep rate and applied layer), it helps to reduce the FID sensitivity to the dropout hyperparameter choice.  
\begin{table}[t!]
	\vspace{-0.5em}
	\setlength{\tabcolsep}{0.2em} 
	\renewcommand{\arraystretch}{1.1}
	\centering
	\caption{Median FIDs (across five random runs) of PA together with dropout applied on different intermediate layers with the keep rate $0.7$ and on CIFAR10 and CELEBA-HQ.} \label{table:pa_dropout}
	\begin{tabular}{l|c|c|c|ccc|c} 
		\rowcolor{verylightgray}
		& & &  & \multicolumn{3}{c}{\footnotesize{}{\text{-$\mathtt{Dropout}$}~\cite{JMLR:v15:srivastava14a}}} & \tabularnewline 
		\rowcolor{verylightgray}
		\multirow{-2}{*}{	\footnotesize{}{\text{$\mathtt{Method}$}}} & 		\multirow{-2}{*}{	\footnotesize{}{\text{$\mathtt{Dataset}$}}}& \multirow{-2}{*}{\footnotesize{}{\text{$\mathtt{PA}$}}} & \multirow{-2}{*}{\footnotesize{}{\text{$\mathtt{GAN}$}}} & 	
		\footnotesize{}{\text{$\mathtt{feat_{N/8}}$}} &
		\footnotesize{}{\text{$\mathtt{feat_{N/4}}$}} &
		\footnotesize{}{\text{$\mathtt{feat_{N/2}}$}}  &  \multirow{-2}{*}{\footnotesize{}{\text{$\overline{\Delta\mathtt{PA}}$}}} \tabularnewline  
 & & 	   	\footnotesize{}{\text{\xmark }} & \text{\footnotesize${26.0}$}  & \text{\footnotesize${24.4}$} & \text{\footnotesize${22.1}$} & \text{\footnotesize${25.6}$}  	& \text{\footnotesize${2.0}$}\tabularnewline  
	\multirow{-2}{*}{	\footnotesize{}{\text{$\mathtt{SN\text{ }DCGAN\text{ - }NS\text{ }Loss}$}}} 	& & \footnotesize{}{\text{$\mathtt{feat_{N/8}}$}} & \text{\footnotesize${22.6}$}  & \text{\footnotesize${21.3}$} & \text{\footnotesize${\mathbf{20.6}}$} & \text{\footnotesize${22.5}$}& \text{\footnotesize${1.1}$} \tabularnewline 	 \arrayrulecolor{verylightgray}	 	 \cline{3-8}  \arrayrulecolor{verylightgray}
&	 & 	   	\footnotesize{}{\text{\xmark }} & \text{\footnotesize${18.8}$}  &  \text{\footnotesize${-}$} & \text{\footnotesize${16.2}$} & \text{\footnotesize${17.1}$} 	& \text{\footnotesize${2.2}$} \tabularnewline  
		\multirow{-2}{*}{	\footnotesize{}{\text{$\mathtt{SA\text{ }GAN\text{ }(sBN)\text{ - }Hinge\text{ }Loss}$}}} &\multirow{-4}{*}{	\footnotesize{}{\text{$\mathtt{CIFAR10}$}}}   & \footnotesize{}{\text{$\mathtt{feat_{N/4}}$}} & \text{\footnotesize${16.3}$}  &  \text{\footnotesize${-}$} & \text{\footnotesize$\mathbf{15.6}$} & \text{\footnotesize${15.7}$}& \text{\footnotesize${0.7}$} \tabularnewline 
		\arrayrulecolor{verylightgray}\hline 
		& & 	\footnotesize{}{\text{\xmark }} & \text{\footnotesize${24.3}$}  &  \text{\footnotesize${-}$} & \text{\footnotesize${24.0}$} & \text{\footnotesize$-$}& \text{\footnotesize${0.3}$} \tabularnewline 
			\multirow{-2}{*}{	\footnotesize{}{\text{$\mathtt{SN\text{ }DCGAN\text{ }\text{ - }NS\text{ }Loss}$}}} &\multirow{-2}{*}{	\footnotesize{}{\text{$\mathtt{CELEBA\text{-}HQ}$}}}   & \footnotesize{}{\text{$\mathtt{feat_{N/8}}$}} & \text{\footnotesize${18.8}$}  &  \text{\footnotesize${-}$} & \text{\footnotesize$\mathbf{18.1}$} & \text{\footnotesize$-$}& \text{\footnotesize${0.7}$} \tabularnewline 
			\arrayrulecolor{verylightgray}\hline 
	&	& \footnotesize{}{\text{$\overline{\Delta\mathtt{PA}}$}} & \text{\footnotesize${3.8}$} &\text{\footnotesize${3.1}$} & \text{\footnotesize${2.7}$} & \text{\footnotesize${2.3}$}  &	\tabularnewline 
	\end{tabular}	
\end{table}

\begin{table}[t!]
	\centering
	\caption{Median FIDs (across five random runs) of PA together with dropout applied on different intermediate layers and keep rates on CIFAR10 with SN DCGAN.}\label{tab:dropoutpa}
	\begin{tabular}{c|ccccccccc} 
		\rowcolor{verylightgray}
		&\footnotesize{}{\text{\footnotesize$\mathtt{Dropout}$}} &\multicolumn{2}{c}{\footnotesize{}{\text{\footnotesize$\mathtt{input(N)}$}}}  & \multicolumn{2}{c}{\footnotesize{}{\text{\footnotesize$\mathtt{feat_{N/2}}$}}} & \multicolumn{2}{c}{\footnotesize{}{\text{\footnotesize$\mathtt{feat_{N/4}}$}}} & \multicolumn{2}{c}{\footnotesize{}{\text{\footnotesize$\mathtt{feat_{N/8}}$}}} \tabularnewline 
		\rowcolor{verylightgray}
		\multirow{-2}{*}{}	&\footnotesize{}{\text{\footnotesize$\mathtt{PA} (\mathtt{feat_{N/8}})$}} & \footnotesize{}{\text{\xmark }}	& \footnotesize{}{\text{\cmark }} & \footnotesize{}{\text{\xmark }}	& \footnotesize{}{\text{\cmark }} & \footnotesize{}{\text{\xmark }}	& \footnotesize{}{\text{\cmark }} & \footnotesize{}{\text{\xmark }}	& \footnotesize{}{\text{\cmark}} \tabularnewline 		
		\multirow{3}{*}{\text{\footnotesize$\mathtt{Keep\text{ }Rate}$}}	& \footnotesize{}{\text{$0.9$}} & \text{\footnotesize$26.4$} & \text{\footnotesize$22.6$}  &  \text{\footnotesize$25.1$} &\text{\footnotesize$21.9$} &  \text{\footnotesize$23.4$} & \text{\footnotesize$21.2$} & \text{\footnotesize$24.6$} &  \text{\footnotesize$21.6$} \tabularnewline 
		& \footnotesize{}{\text{$0.7$}} & \text{\footnotesize$28.0$} & \text{\footnotesize$22.9$} &  \text{\footnotesize$25.6$} & \text{\footnotesize$21.3$}& \text{\footnotesize$\underline{22.1}$} & \text{\footnotesize$\mathbf{20.6}$} & \text{\footnotesize$24.4$} &\text{\footnotesize$22.5$} \tabularnewline 
		& \footnotesize{}{\text{$0.5$}} & \text{\footnotesize$27.1$} & \text{\footnotesize$23.1$} &  \text{\footnotesize$25.9$} & \text{\footnotesize$22.3$} & \text{\footnotesize$23.1$} & \text{\footnotesize$21.2$} & \text{\footnotesize$24.0$} & \text{\footnotesize$22.1$} \tabularnewline 
		\rowcolor{verylightgray}
		\multicolumn{2}{c}{\footnotesize{}{\text{$\Delta\mathtt{PA}$}}} &\multicolumn{2}{c}{\footnotesize{}{\text{$4.5$}}}  & \multicolumn{2}{c}{\footnotesize{}{\text{$3.7$}}} & \multicolumn{2}{c}{\footnotesize{}{\text{$1.9$}}} & \multicolumn{2}{c}{\footnotesize{}{\text{$2.3$}}}
	\end{tabular}
\end{table}

\begin{table}[t!]
	\setlength{\tabcolsep}{0.22em} 
	\renewcommand{\arraystretch}{1.1}
	\centering
	\caption{Median FIDs (across five random runs) of different learning rates ($\mathtt{TTURs}$) on CIFAR10 with $\mathtt{SN\text{ }DCGAN}$.
		Italic and bold denotes the best FIDs w/o and with PA respectively, underline denotes the default learning rate setting of $\mathtt{SN\text{ }DCGAN}$.} \label{table:ttur} \vspace{0.5em}
	\begin{tabular}{c|c|cccc|c} 
		\rowcolor{verylightgray}
		\text{\backslashbox{$\mathtt{lr_g}$}{$\mathtt{lr_d}$}} & \footnotesize{}{\text{$\mathtt{PA\text{ }(feat_{N/8})}$}} &  \footnotesize{}{ \text{\footnotesize$ 10^{-4}$}} & \footnotesize{}{ \text{\footnotesize$2\times 10^{-4}$}} & \footnotesize{}{ \text{\footnotesize$4\times 10^{-4}$}} & \footnotesize{}{ \text{\footnotesize$ 10^{-3}$}}  & \footnotesize{}{\text{$\overline{\Delta\mathtt{PA}}$}}\tabularnewline 
		
		\multirow{2}{*}{\text{\footnotesize$10^{-4}$}} &\footnotesize{}{\xmark} &\text{\footnotesize$27.0$} & \text{\footnotesize$25.8$}  & \text{\footnotesize$\mathit{25.3}$} & \text{\footnotesize$27.0$} & \multirow{2}{*}{\text{\footnotesize$3.5$}}\tabularnewline 	
		&\footnotesize{}{\cmark} & \text{\footnotesize$23.3$} & \text{\footnotesize$\mathbf{22.2}$} & \text{\footnotesize$22.6$}& \text{\footnotesize$22.9$} & \tabularnewline 
		\arrayrulecolor{gray}	\hline \arrayrulecolor{verylightgray}
		
		\multirow{2}{*}{\text{\footnotesize$2\times 10^{-4}$}} &\footnotesize{}{\xmark} &\text{\footnotesize$26.7$} & \text{\footnotesize\underline{$26.0$}}  & \text{\footnotesize$26.2$} & \text{\footnotesize$27.2$} & \multirow{2}{*}{\text{\footnotesize$3.1$}} \tabularnewline 	
		&\footnotesize{}{\cmark} & \text{\footnotesize$24.8$} & \text{\footnotesize\underline{$22.6$}} & \text{\footnotesize$22.3$}& \text{\footnotesize$24.0$}& \tabularnewline 
		\arrayrulecolor{gray}	\hline \arrayrulecolor{verylightgray}
		
		\multirow{2}{*}{\text{\footnotesize$4\times 10^{-4}$}} &\footnotesize{}{\xmark} &\text{\footnotesize$28.7$} & \text{\footnotesize$26.1$}  & \text{\footnotesize$26.3$} & \text{\footnotesize$28.2$}& \multirow{2}{*}{\text{\footnotesize$3.6$}}\tabularnewline 	
		&\footnotesize{}{\cmark} & \text{\footnotesize$24.7$} & \text{\footnotesize$23.3$} & \text{\footnotesize$22.9$}& \text{\footnotesize$24.2$}& \tabularnewline 
		\arrayrulecolor{gray}	\hline \arrayrulecolor{verylightgray}
		
		\multirow{2}{*}{\text{\footnotesize$10^{-3}$}} &\footnotesize{}{\xmark} &\text{\footnotesize$28.5$} & \text{\footnotesize$27.0$}  & \text{\footnotesize$26.4$} & \text{\footnotesize$27.4$} &  \multirow{2}{*}{\text{\footnotesize$2.9$}} \tabularnewline 	
		&\footnotesize{}{\cmark} & \text{\footnotesize$25.7$} & \text{\footnotesize$23.6$} & \text{\footnotesize$23.4$}& \text{\footnotesize$25.0$}& 
	\end{tabular}	
	\vspace{-0.7em}
\end{table}

\subsection{Ablation Study on Learning Rates}\label{Ssubsec:ttur}
Table~\ref{table:ttur} compares the performance achieved by using different learning rate configurations. The improvement achieved by $\mathtt{PA}$ is consistent across different settings ($\sim3$ FID points), showing its robustness to different update rules. Compared to the best performing TTUR, $\mathtt{PA}$ reduces the FID faster over iterations (see Figure~\ref{fig:scheduling}) without requiring extra hyperparameter search for the best update rule. 

Table~\ref{table:ttur} has also shown a stable FID performance of $\mathtt{SN\:DCGAN}$ with the generator learning rate $\mathtt{lr_g}=2\times 10^{-4}$ and the discriminator learning rate $\mathtt{lr_d}\in\{10^{-4},2\times 10^{-4},4\times10^{-4}\}$. With this identification, we fix $\mathtt{lr_g}=2\times 10^{-4}$ and reuse the progression scheduling to adaptively reduce $\mathtt{lr_d}$ from $4\times10^{-4}$ to $10^{-4}$ with the learning rate decay of $0.8$ (in our experiments the best performing learning rate decay among $\{0.99,0.95,0.9,0.8,0.7\}$). 
Figure~\ref{fig:scheduling} shows the effectiveness of progression scheduling in assisting both the learning rate adaptation and progressive augmentation for an improved performance. $\mathtt{PA}$ outperforms learning rate adaptation as well as the tuned $\mathtt{TTUR}$~\cite{heuselttur2017} , i.e. FID $22.6$ vs. $24.0$ vs. $25.3$. Its combination with $\mathtt{Dropout}$ delivers the best performance in this experiment, i.e., $20.6$.

\begin{figure*}[t!]
	\begin{center}
		\includegraphics[width=0.8\textwidth]{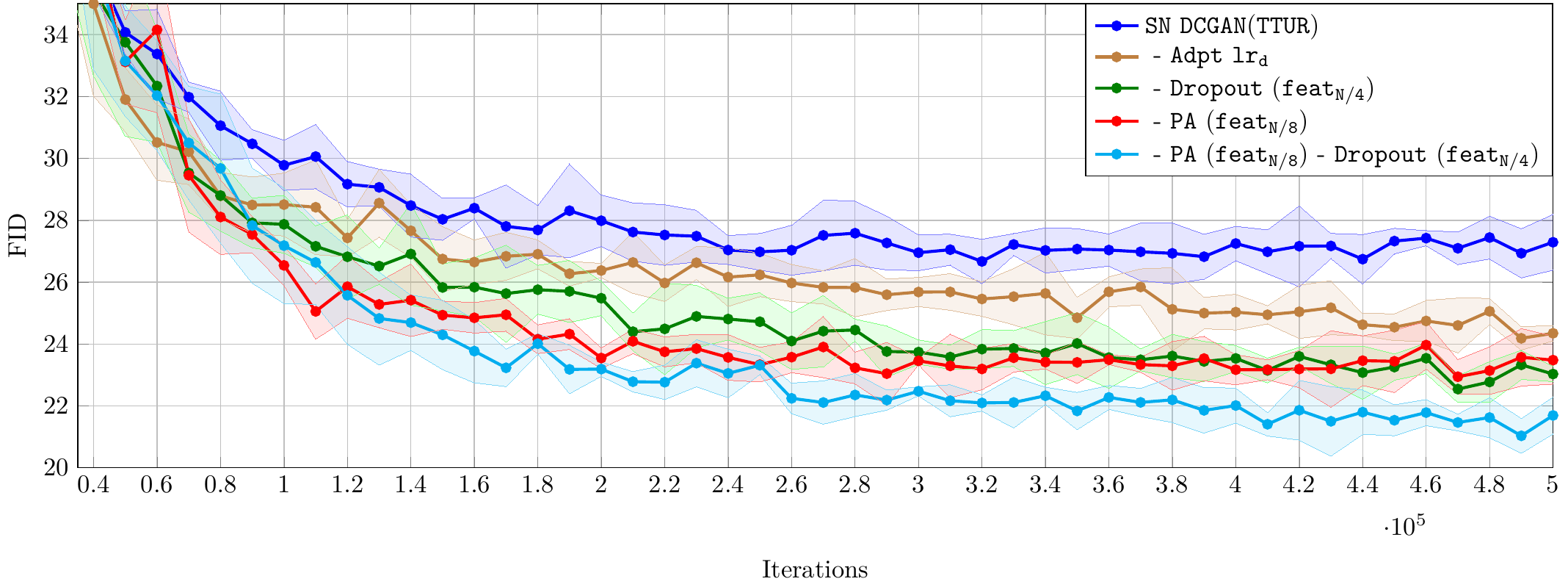}
	\end{center}
	\caption{\label{fig:scheduling}FID learning curves (mean FIDs with one standard deviation across five random runs) of $\mathtt{PA}$, $\mathtt{TTUR}$, adaptive learning rate and $\mathtt{Dropout}$ on CIFAR10 with $\mathtt{SN\text{ }DCGAN}$.}
\end{figure*}

\section{Effectiveness of PA as a Regularizer}\label{sup_sec_toy_example}

Here we exploit progressive augmentation on a toy classification task to empirically illustrate its regularization benefits discussed in Sec.~3 of the main paper. Specifically, we focus on binary classification task taking the alike Cat and Dog images from CIFAR10~\cite{Cifar10_Krizhevsky09learningmultiple}, which represent the TRUE (real) and FAKE (synthetic) data samples, and train the discriminator network of SN DCGAN with the cross-entropy loss to tell them apart. Figure~\ref{fig:toy_d_loss} depicts the discriminator loss ($D$ loss) behaviour over iterations on the training and test sets. It shows that the discriminator very quickly becomes over-confident on the training set and that overfitting takes place after $1$\unit{k} iterations.

In order to regularize the discriminator we exploit the proposed progressive augmentation ($\mathtt{PA}$), augmenting either the input - $\mathtt{PA}\text{ }(\mathtt{input})$, or its features with spatial dimension $N/8$ - $\mathtt{PA}\text{ }(\mathtt{feat_{N/8}})$, where $N$ is the input image dimension. For a comparison purpose, we also experiment with the $\mathtt{Dropout}$~\cite{JMLR:v15:srivastava14a} regularization applied on $\mathtt{feat_{N/4}}$ layer with the keep rate $0.7$ (the best performing rate in our experiments). Both techniques resort to random variables for regularization. The former randomly removes features, while the latter augments them with additional random bits and adjusts accordingly the class label. 
In contrast to $\mathtt{Dropout}$, $\mathtt{PA}$ exhibits a long lasting regularization effect by means of progression. Each rise of $D$ loss coinciding with an iteration at which the augmentation level increases (every $2$\unit{k} iterations) and then gradually reduces after the discriminator timely adapts to the new bit. %
At the level one augmentation, both $\mathtt{PA}\text{ }(\mathtt{input})$ and $\mathtt{PA}\text{ }(\mathtt{feat_{N/8}})$ start from the similar overfitting stage. Combining the bit $s$ directly with high-level features eases checksum computation. As a result, the $D$ loss of $\mathtt{PA}\text{ }(\mathtt{feat_{N/8}})$ reduces faster, but making its future task more difficult due to overfitting to the previous augmentation level. On the other hand,  $\mathtt{PA}\text{ }(\mathtt{input})$ let the bits pass through all layers, and thus its adaptation to augmentation progression improves over iterations. In the end, both $\mathtt{PA}\text{ }(\mathtt{input})$ and $\mathtt{PA}\text{ }(\mathtt{feat_{N/8}})$ lead to similar regularization effect. In addition, we compare $\mathtt{PA}$ with the $\mathtt{Reinit.}$ baseline, where every $2$\unit{k} iterations all weights are reinitialized with Xavier initialization \cite{GlorotAISTATS2010}. Compared to $\mathtt{PA}$, using $\mathtt{Reinit.}$ strategy leads to longer adaptation time (the $D$ loss decay is much slower), potentially providing non-informative signal to the generator and thus slowing down the training.

In Figure~\ref{fig:toy_hist} we explore the stochastic nature of $\mathtt{Dropout}$ and $\mathtt{PA}$. Each realization of the dropout mask or the augmentation bit sequence $\bvec{\vs}$ changes the loss function landscape, varying its gradient with respect to the synthetic sample (i.e. the Dog class in this case). With the same experimental setup, we now assess the correlation of the gradients based on the first four eigenvalues of their correlation matrix - $\lambda_i$, $i=0,\dots,3$, i.e. computing the averaged square roots of their ratios $\bar{\gamma}\stackrel{\Delta}{=}\frac{1}{3}\sum_{i=1}^{3}\sqrt{\lambda_0/\lambda_i}$. Figure~\ref{fig:toy_hist} depicts the histograms of $\bar{\gamma}$ among $10^3$ instances. %
$\mathtt{PA}$ has more instances with smaller $\bar{\gamma}$ in comparison to $\mathtt{Dropout}$, indicating a more diverse set of gradients, exploitable by the generator to approach the data distribution. In contrast to $\mathtt{Dropout}$, in $\mathtt{PA}$ the augmentation random bits determine the target class in binary classification and the discriminator is trained to comprehend $\bvec{\vs}$ together with $\bvec{\vx}$, leading to the richer loss function landscape. %
Between input and feature space augmentation, the former yields more diverse gradients than the latter as $\bvec{\vs}$ is passed through all layers. 

\begin{figure}[t!]
	\begin{center}
		\includegraphics[width=\textwidth]{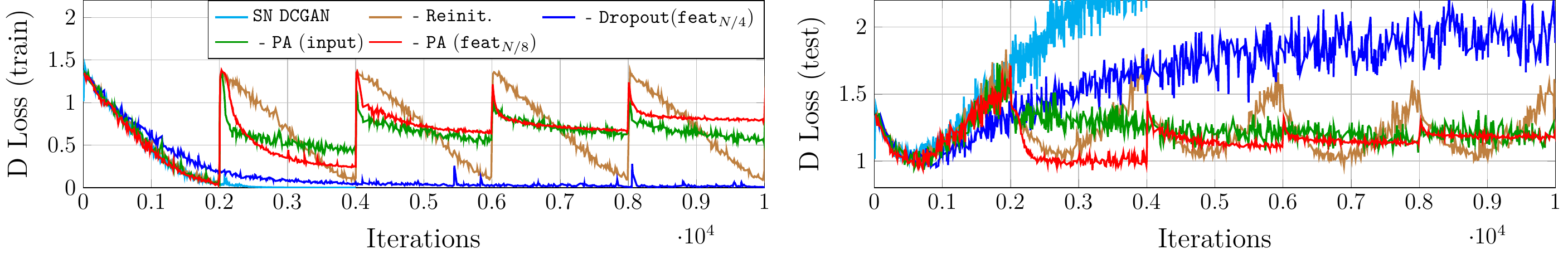}
	\end{center}
	\vspace{-1em}
	\caption{\label{fig:toy_d_loss} Behaviour of the discriminator loss ($D$ loss) with and w/o $\mathtt{PA}$ and in comparison to $\mathtt{Dropout}$, using the $D$ architecture of SN DCGAN. See Sec.~\ref{sup_sec_toy_example} for details.} 
	\vspace{-1em}
\end{figure}
\begin{figure}[t!]%
	\centering
	\includegraphics[width=.8\textwidth]{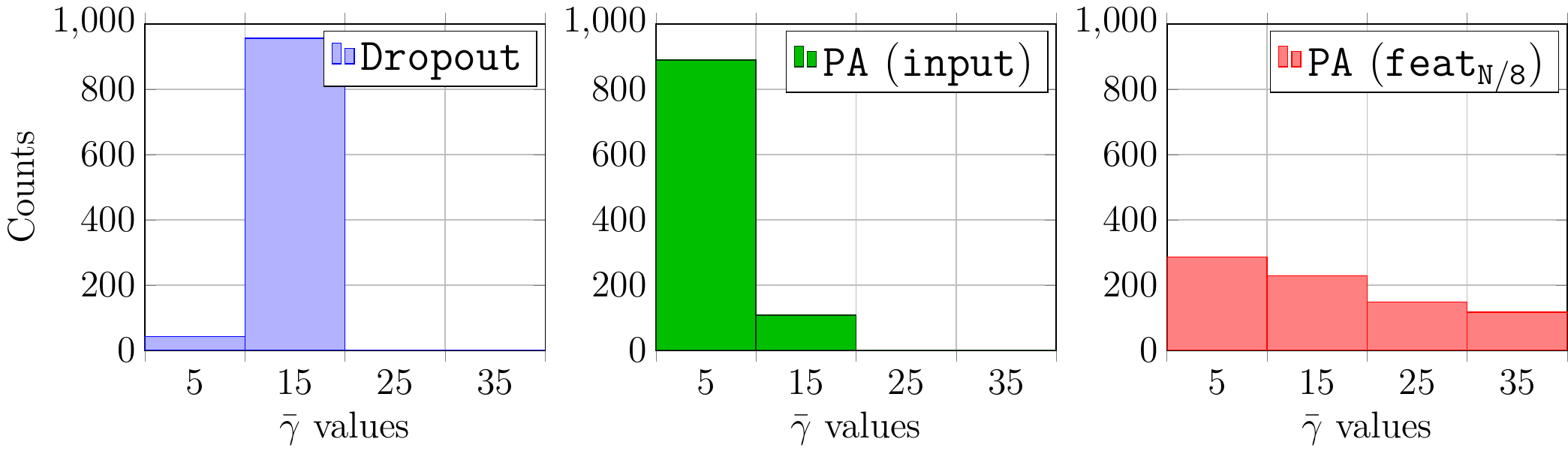}
	\caption{Histograms of averaged square roots of eigenvalue ratios computed from gradient correlation matrices for $\mathtt{PA}$ and $\mathtt{Dropout}$. Smaller correlation values indicate a more diverse set of gradients exploitable by the generator to approach the data distribution. See Sec.~\ref{sup_sec_toy_example} for details.}
	\label{fig:toy_hist}
\end{figure}

\section{Exemplar Synthetic Samples}\label{sec:syn samples}
Figure~\ref{Sfig:images} shows a set of synthetic samples that are outcomes of GAN training with and without $\mathtt{PA}$. $\mathtt{PA}$ not only improves sample quality and variation, but also sensibly navigates the image manifold through latent space interpolation.

\begin{figure*}	
	\centering
	\begin{subfigure}{\textwidth}
	\includegraphics[width = \linewidth]{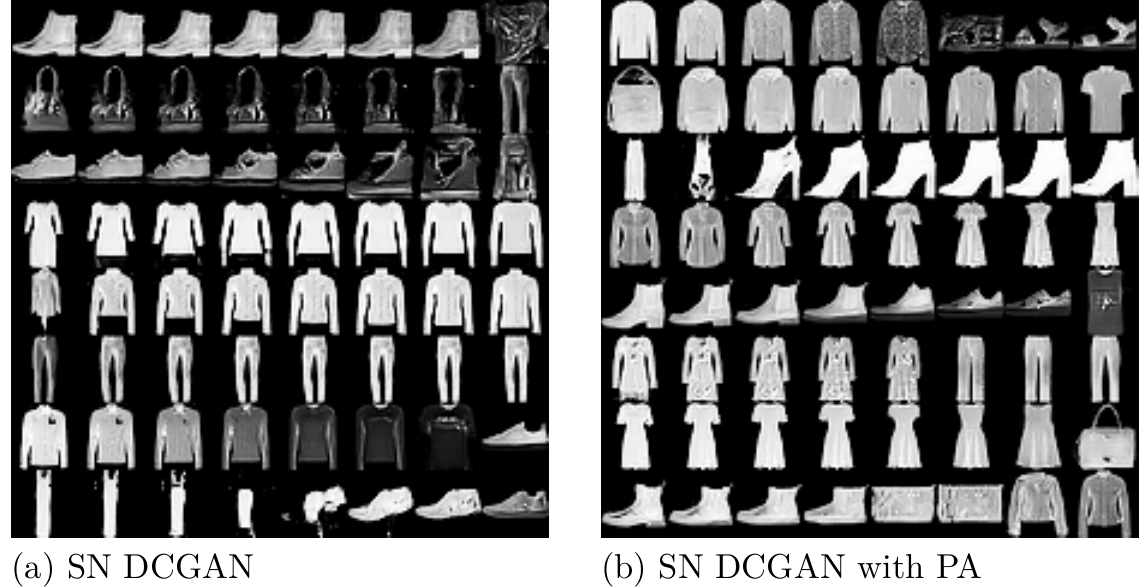}
	\end{subfigure}	
	\begin{subfigure}{\textwidth}
	\includegraphics[width = \linewidth]{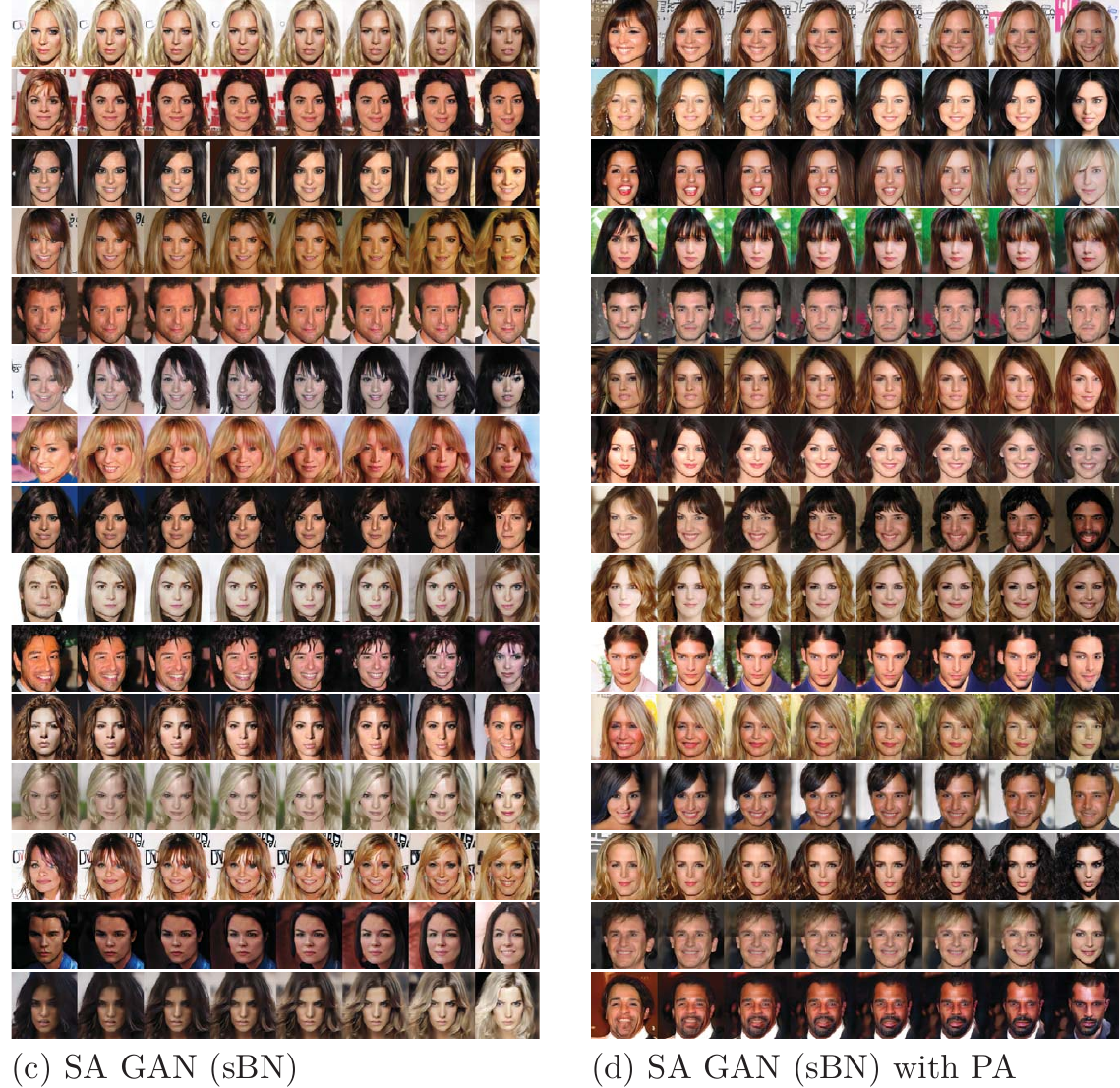}
	\end{subfigure}	
	\caption{Synthetic samples from training SN GAN on Fashion-MNIST ($28\times 28$) and SA GAN (sBN) on CELEBA-HQ ($128\times 128$) with and without using PA. In all cases, i.e., (a), (b), (c) and (d), the eight images per row are generated through polar-interpolation between two randomly sampled $\bvec{\vz}_1$ and $\bvec{\vz}_2$.}\label{Sfig:images}
\end{figure*}

\section{Evaluation with Other Performance Measures} \label{kid-is}
In addition to FID, here we measure the quality of synthetic samples by means of kernel inception distance (KID)~\cite{Binkowski2016MMDGAN} and inception score (IS)~\cite{Theis2016a}, see Tables~\ref{table:kid_is_one} and~\ref{table:kid_is_pa_gp_dropout} which correspond to Tables~1 and~3 in the main paper. The evaluation framework setup is the same as that with FID and follows~\cite{LucicEqualGANs,Kurach2018GANlandscape}. For Fashion-MNIST and CELEBA-HQ, IS computed from the pre-trained Inception network is not meaningful and thus omitted. Overall, the obtained results show consistent observations with those that are made in Sec.~4 of the main paper based on the FID measure.
\begin{table}[t!]
	\setlength{\tabcolsep}{0.15em} 
	\renewcommand{\arraystretch}{1}
	\centering
	\caption{KID/IS improvements with $\mathtt{PA}$ across different datasets and network architectures, in accordance with Table~1 in the main paper.} \label{table:kid_is_one} 
	\hspace{-1em}
	\begin{minipage}{0.6\textwidth}
		\caption*{KID}
		\begin{tabular}{l|c|cccc|c} 
		\rowcolor{verylightgray}
		\footnotesize{}{\text{$\mathtt{Method}$}}	& \footnotesize{}{\text{$\mathtt{PA}$}}& 
		\footnotesize{}{\text{$\mathtt{F\text{-}MNIST}$}} & \footnotesize{}{\text{$\mathtt{CIFAR10}$}} & \footnotesize{}{\text{$\mathtt{CELEBA\text{-}HQ}$}} & \footnotesize{}{\text{$\mathtt{T\text{-}ImageNet}$}}&\footnotesize{}{\text{$\overline{\Delta\mathtt{PA}}$}} \tabularnewline 	
		
		\text{$\mathtt{SN\text{ }DCGAN}$} & \footnotesize{}{\xmark}& \text{\footnotesize$0.004$} & \text{\footnotesize$0.016$} & \text{\footnotesize$0.011$} & \text{\footnotesize-} & \multirow{3}{*}{\text{\footnotesize$0.003$}} \tabularnewline 	
		\text{$\mathtt{NS \text{ } Loss}$} & \text{$\mathtt{input}$} & \text{\footnotesize$\mathbf{0.002}$} &  \text{\footnotesize$\mathbf{0.013}$} & \text{\footnotesize${0.007}$} & \text{\footnotesize-} \tabularnewline 
		\footnotesize{}\text{\cite{miyato2018spectral}}  & \text{$\mathtt{feat}$} & \text{\footnotesize$\mathbf{0.002}$} &  \text{\footnotesize$\mathbf{0.013}$} & \text{\footnotesize$\mathbf{0.005}$} & \text{\footnotesize-} \tabularnewline 
		
		\arrayrulecolor{gray}	\hline \arrayrulecolor{verylightgray}
		
		\multirow{1}{*}{\text{$\mathtt{SA\text{ }GAN\text{ }(sBN)}$}}	& \footnotesize{}{\xmark}& \text{\footnotesize-} & \text{\footnotesize$0.011$} & \footnotesize{}{$0.006$} &\footnotesize{}{$0.035$} &\multirow{3}{*}{\text{\footnotesize$0.002$}}   \tabularnewline 
		\text{$\mathtt{Hinge \text{ } Loss}$}  & \text{$\mathtt{input}$} & \text{\footnotesize-}  & \text{\footnotesize$\mathbf{0.008}$} & \footnotesize{}{$\mathbf{0.004}$} & \text{\footnotesize$\mathbf{0.033}$}  \tabularnewline 
		\footnotesize{}\text{\cite{Zhang_SAGAN18}} & \text{$\mathtt{feat}$} & \text{\footnotesize -}  & \text{\footnotesize$0.009$} & \footnotesize{}{$\mathbf{0.004}$} & \text{\footnotesize$\mathbf{0.033}$} 
	\end{tabular}	
	\end{minipage}
\hspace{3.5em}
\begin{minipage}{0.3\textwidth}
	\caption*{IS}
	\begin{tabular}{||cc|c} 
	\rowcolor{verylightgray}
	 \footnotesize{}{\text{$\mathtt{CIFAR10}$}} &  \footnotesize{}{\text{$\mathtt{T\text{-}ImageNet}$}}&\footnotesize{}{\text{$\overline{\Delta\mathtt{PA}}$}} \tabularnewline 	
	
 \text{\footnotesize$7.6$} & \text{\footnotesize-} & \multirow{3}{*}{\text{\footnotesize$0.2$}} \tabularnewline 	
 \text{\footnotesize$\mathbf{7.8}$} &  \text{\footnotesize-} \tabularnewline 
 \text{\footnotesize$\mathbf{7.8}$}  & \text{\footnotesize-} \tabularnewline 
	
	\arrayrulecolor{gray}	\hline \arrayrulecolor{verylightgray}
	
\text{\footnotesize$8.4$} &\footnotesize{}{$8.8$} &\multirow{3}{*}{\text{\footnotesize$0.3$}}   \tabularnewline 
 \text{\footnotesize$\mathbf{8.7}$}  & \text{\footnotesize$9.1$}  \tabularnewline 
 \text{\footnotesize$8.6$}  & \text{\footnotesize$\mathbf{9.2}$} 
\end{tabular}	
\end{minipage}
\end{table}

\begin{table}[t!]
	\vspace{-1em}
	\setlength{\tabcolsep}{0.2em} 
	\renewcommand{\arraystretch}{1.1}
	\centering
	\caption{KIDs/ISs of PA, different regularization techniques and their combinations on CIFAR10, in according with Table 3 in the main paper.} \label{table:kid_is_pa_gp_dropout}
	\begin{minipage}{0.9\textwidth}
		\centering
		\caption*{KID}
	\begin{tabular}{l|c|c|ccccc|c} 		
	\rowcolor{verylightgray}
	& & & \footnotesize{}{\text{-$\mathtt{Label\text{ }smooth.}$}}  &\footnotesize{}{\text{-$\mathtt{GP}$}} & \footnotesize{}{\text{-$\mathtt{GP_{zero\text{-}cent}}$}} & \footnotesize{}{\text{-$\mathtt{Dropout}$}} &\footnotesize{}{\text{-$\mathtt{SS}$}}  & \tabularnewline 
	\rowcolor{verylightgray}
	\multirow{-2}{*}{	\footnotesize{}{\text{$\mathtt{Method}$}}} & \multirow{-2}{*}{\footnotesize{}{\text{$\mathtt{PA}$}}} & \multirow{-2}{*}{\footnotesize{}{\text{$\mathtt{GAN}$}}} & 	\footnotesize{}{\text{\cite{SalimansNIPS2016}}} &  \footnotesize{}{\cite{gulrajani_NIPS2017}} & \footnotesize{}{ \cite{Roth_NIPS2017}} &
	\footnotesize{}{\cite{JMLR:v15:srivastava14a}} & \footnotesize{}{\cite{ChenSS2019}} &  \multirow{-2}{*}{\footnotesize{}{\text{$\overline{\Delta\mathtt{PA}}$}}} \tabularnewline  
	\footnotesize{}{\text{$\mathtt{SN\text{ }DCGAN}$}} & 	   	\footnotesize{}{\text{\xmark }} & \text{\footnotesize${0.016}$}  & \text{\footnotesize${0.016}$} & \text{\footnotesize${0.018}$} &\text{\footnotesize${0.017}$}  & \text{\footnotesize${0.013}$} &  \text{\footnotesize${-}$}  &	\tabularnewline  
	\footnotesize{}{\text{$\mathtt{NS\text{ }Loss}$}} & \footnotesize{}{\text{$\mathtt{feat}$}} & \text{\footnotesize${0.013}$}  & \text{\footnotesize${0.014}$}  & \text{\footnotesize${0.014}$} & \text{\footnotesize${0.014}$} & \text{\footnotesize${\mathbf{0.012}}$} &\text{\footnotesize${-}$}  &
	\multirow{-2}{*}{\footnotesize${0.003}$}\tabularnewline 
	\arrayrulecolor{verylightgray}	\hline 
	\footnotesize{}{\text{$\mathtt{SA\text{ }GAN\text{ }(sBN)}$}}	 & 	   	\footnotesize{}{\text{\xmark }} & \text{\footnotesize${0.011}$}  & \text{\footnotesize$-$} & \text{\footnotesize${0.010}$} &\text{\footnotesize${0.010}$} & \text{\footnotesize${0.008}$} &\text{\footnotesize${0.008}$} & \tabularnewline  
	\footnotesize{}{\text{$\mathtt{Hinge\text{ }Loss}$}} & \footnotesize{}{\text{$\mathtt{feat}$}} & \text{\footnotesize${0.009}$}  & \text{\footnotesize${-}$}  & \text{\footnotesize${0.008}$} & \text{\footnotesize${0.008}$} &  \text{\footnotesize${0.008}$} &\text{\footnotesize$\mathbf{0.007}$} &
	\multirow{-2}{*}{\footnotesize$0.001$}\tabularnewline 
	\arrayrulecolor{verylightgray}\hline 
	& \footnotesize{}{\text{$\overline{\Delta\mathtt{PA}}$}} & \text{\footnotesize${0.003}$} & \text{\footnotesize${0.002}$}& \text{\footnotesize${0.003}$} & \text{\footnotesize${0.003}$} & \text{\footnotesize${0.001}$} &  \text{\footnotesize${0.001}$} &\tabularnewline 
\end{tabular}	
	\end{minipage}
\\
	\begin{minipage}{0.9\textwidth}
		\centering
				\caption*{IS}
	\begin{tabular}{l|c|c|ccccc|c} 
		\rowcolor{verylightgray}
		& & & \footnotesize{}{\text{-$\mathtt{Label\text{ }smooth.}$}}  &\footnotesize{}{\text{-$\mathtt{GP}$}} & \footnotesize{}{\text{-$\mathtt{GP_{zero\text{-}cent}}$}} & \footnotesize{}{\text{-$\mathtt{Dropout}$}} &\footnotesize{}{\text{-$\mathtt{SS}$}}  & \tabularnewline 
		\rowcolor{verylightgray}
		\multirow{-2}{*}{	\footnotesize{}{\text{$\mathtt{Method}$}}} & \multirow{-2}{*}{\footnotesize{}{\text{$\mathtt{PA}$}}} & \multirow{-2}{*}{\footnotesize{}{\text{$\mathtt{GAN}$}}} & 	\footnotesize{}{\text{\cite{SalimansNIPS2016}}} &  \footnotesize{}{\cite{gulrajani_NIPS2017}} & \footnotesize{}{ \cite{Roth_NIPS2017}} &
		\footnotesize{}{\cite{JMLR:v15:srivastava14a}} & \footnotesize{}{\cite{ChenSS2019}} &  \multirow{-2}{*}{\footnotesize{}{\text{$\overline{\Delta\mathtt{PA}}$}}} \tabularnewline  
		\footnotesize{}{\text{$\mathtt{SN\text{ }DCGAN}$}} & 	   	\footnotesize{}{\text{\xmark }} & \text{\footnotesize${7.6}$}  & \text{\footnotesize${7.5}$} & \text{\footnotesize${7.5}$} &\text{\footnotesize${7.5}$}  & \text{\footnotesize${7.9}$} &  \text{\footnotesize${-}$}  &	\tabularnewline  
		\footnotesize{}{\text{$\mathtt{NS\text{ }Loss}$}} & \footnotesize{}{\text{$\mathtt{feat}$}} & \text{\footnotesize${7.8}$}  & \text{\footnotesize${7.7}$}  & \text{\footnotesize${7.7}$} & \text{\footnotesize${7.7}$} & \text{\footnotesize${\mathbf{7.9}}$} &\text{\footnotesize${-}$}  &
		\multirow{-2}{*}{\footnotesize${{0.2}}$}\tabularnewline 
		\arrayrulecolor{verylightgray}	\hline 
		\footnotesize{}{\text{$\mathtt{SA\text{ }GAN\text{ }(sBN)}$}}	 & 	   	\footnotesize{}{\text{\xmark }} & \text{\footnotesize${8.4}$}  & \text{\footnotesize$-$} & \text{\footnotesize${8.5}$} &\text{\footnotesize${8.5}$} & \text{\footnotesize${8.7}$} &\text{\footnotesize${8.6}$} & \tabularnewline  
		\footnotesize{}{\text{$\mathtt{Hinge\text{ }Loss}$}} & \footnotesize{}{\text{$\mathtt{feat}$}} & \text{\footnotesize${8.6}$}  & \text{\footnotesize${-}$}  & \text{\footnotesize${8.6}$} & \text{\footnotesize${8.7}$} &  \text{\footnotesize${8.7}$} &\text{\footnotesize$\mathbf{8.8}$} &
		\multirow{-2}{*}{\footnotesize$0.1$}\tabularnewline 
		\arrayrulecolor{verylightgray}\hline 
		& \footnotesize{}{\text{$\overline{\Delta\mathtt{PA}}$}} & \text{\footnotesize${0.2}$} & \text{\footnotesize${0.2}$}& \text{\footnotesize${0.2}$} & \text{\footnotesize${0.2}$} & \text{\footnotesize${0.0}$} &  \text{\footnotesize${0.2}$} &\tabularnewline 
	\end{tabular}	
\end{minipage}
\end{table}

\section{Network Architectures and Hyperparameter Settings}\label{sec:networks}
In this work we exploit the implementation provided by~\cite{LucicEqualGANs, Kurach2018GANlandscape}\footnote{\url{https://github.com/google/compare_gan}} and~\cite{Zhang_SAGAN18}\footnote{\url{https://github.com/brain-research/self-attention-gan}}. For the experiments, we run on single GPU (Nvidia Titan X).
\subsection{Network Architectures}\label{subsec:netarch}

\paragraph{SN DCGAN.}\label{subsec:sndcgan}
Following~\cite{miyato2018spectral} for spectral normalization (SN), we adopt the same architecture as in~\cite{Kurach2018GANlandscape} and present its configuration in Table~\ref{tab_sndcgan}. The input and feature (i.e., $\mathtt{feat_{N/2}}$, $\mathtt{feat_{N/4}}$ and $\mathtt{feat_{N/8}}$) space augmentations respectively take place at the input of the layers with the index $0$, $2$, $4$ and $6$. In case of dropout, it is applied to the same intermediate layers plus the output of the layer $0$. For Table~1 in the main paper, we pick the  $\mathtt{feat_{N/8}}$ for all evaluated datasets, whereas Sec.~\ref{Ssubsec:abl aug level} presents an ablation study on the augmentation space.

\paragraph{SA GAN (sBN).}\label{subsec:sagan}
The ResNet-based discriminator and generator architectures tailored for CIFAR10, CELEBA-HQ and T-ImageNet are presented in Table~\ref{tab_sagan} and~\ref{tab_sagantiny}, respectively. Taking the ResNet architecture in~\cite{gulrajani_NIPS2017} for CIFAR10, in~\cite{Kurach2018GANlandscape} for CELEBA-HQ and~\cite{Brock2019} for IMAGENET as the baseline, we adapt them by adding the SN and self-attention as proposed in~\cite{Zhang_SAGAN18}. For the residual and non-local blocks we use the implementation provided by~\cite{Zhang_SAGAN18}. As we target unsupervised GAN, the conditional batch normalization (BN) used by the generator's residual blocks only takes the input noise vector $\bvec{\vz}$ as the conditioning, namely, self-modulation BN (sBN)~\cite{chen2018on}.

For CIFAR10, we have considered the input and feature (i.e., $\mathtt{feat_{N/2}}$ and $\mathtt{feat_{N/4}}$) space augmentations which respectively take place at the input of the residual blocks with the index $0$, $2$ and $4$, see Table~\ref{tab_sagan}-(a). Note that both residual blocks with the index $3$ and $4$ have their feature maps of dimension ${N/4}$. We experiment with the feature space augmentation on both of them. They differ little in performance, thereby we only report the result of the feature space augmentation at the residual block $4$ in Table 1 of the main paper.

For CELEBA-HQ, we empirically observe that it is beneficial to start from a convolutional layer rather than a residual block at the discriminator. Apart from input and $\mathtt{feat_{N/8}}$ space augmentation reported in Table~1 of the main paper, we have also experimented the other feature space augmentations that take place at the input of each residual block, see Table~\ref{tab_sagancelebahq}. At the spatial dimension $N$, we only report the result of input space augmentation, whereas the feature space augmentation at the first residual block delivers a similar performance. Augmenting the input of the last residual block benefits from the first warm-up mechanism presented in Sec.~\ref{sec:warm-up}, otherwise the discriminator can fail after augmentation progression.

For T-ImageNet, we have experimented with the augmentation space at both the input and $\mathtt{feat_{16}}$ (at the input of the $3$rd residual block) and reported their performance in Table~1 of the main paper. It is beneficial to use the second warm-up mechanism introduced in Sec.~\ref{sec:warm-up}. Comparing with the other datasets, the synthesis quality on T-ImageNet is still poor. Single GPU simulation with $64$ samples per batch is not enough in this case. Large-scale simulation as in~\cite{Brock2019}, though demanding a large amount of resources, would be of interest.

\begin{table*}[t!]
	\centering
	\caption{SN DCGAN.\label{tab_sndcgan}}\hspace{-0.5cm}
	\begin{minipage}[t]{.45\textwidth}
	\centering	
	\subcaption{Discriminator}
\begin{tabular}{ll} 	
			\#	& \text{Configuration per Layer} \\ \hline
			0 &	\text{$3\times3$ stride $1$ SN Conv, $\mathtt{ch=64}$, lReLu} \\\hline
			1 &	\text{$4\times4$ stride $2$ SN Conv, $\mathtt{ch=128}$, lReLu} \\\hline
			2 &	\text{$3\times3$ stride $1$ SN Conv, $\mathtt{ch=128}$, lReLu} \\\hline
			3 &	\text{$4\times4$ stride $2$ SN Conv, $\mathtt{ch=256}$, lReLu} \\\hline
			4 &	\text{$3\times3$ stride $1$ SN Conv, $\mathtt{ch=256}$, lReLu} \\\hline
			5 &	\text{$4\times4$ stride $2$ SN Conv, $\mathtt{ch=512}$, lReLu} \\\hline
			6 &	\text{$3\times3$ stride $1$ SN Conv, $\mathtt{ch=512}$, lReLu} \\\hline
			7 &	\text{SN Linear $1$ output}   \\
			\bottomrule
		\end{tabular}
\end{minipage}
\hspace{0.5cm}
		\begin{minipage}[t]{.45\textwidth}	
			\centering
			\subcaption{Generator}
		\begin{tabular}{l} 	
			\text{Configuration per Layer} \\ \hline
			\text{Linear $h/8\times w/8\times 512$ output, BN, ReLU} \\\hline
			\text{$4\times4$ stride $2$ DeConv, $\mathtt{ch=256}$, BN, ReLU}  \\\hline
			\text{$4\times4$ stride $2$ DeConv, $\mathtt{ch=128}$, BN, ReLU}  \\\hline
			\text{$4\times4$ stride $2$ DeConv, $\mathtt{ch=64}$, BN, ReLU}  \\\hline
			\text{$3\times3$ stride $1$ Deconv, $\mathtt{ch=3}$, Tanh}  \\ 
			\bottomrule
		\end{tabular}
		\end{minipage}	
\end{table*}

\begin{table*}[t!]
	\centering
	\caption{SA GAN for CIFAR10.\label{tab_sagan}}\hspace{-0.5cm}
	\begin{minipage}[t]{.45\textwidth}
		\centering	
		\subcaption{Discriminator}
		\begin{tabular}{ll} 	
			\#	&	\text{Configuration per Layer} \\ \hline
			0	&	\text{ResBlock, down, $\mathtt{ch=128}$ } \\ \hline
			1	&	\text{Non-Local Block ($16\times 16$)} \\ \hline
			2	&	\text{ResBlock, down, $\mathtt{ch=128}$ }  \\\hline
			3	&	\text{ResBlock, $\mathtt{ch=128}$ } \\ \hline
			4	&	\text{ResBlock, $\mathtt{ch=128}$ }\\ \hline
			5	&	\text{ReLU, Global sum pooling} \\ \hline
			6	&	\text{SN Linear $1$ output} \\
			\bottomrule
		\end{tabular}
	\end{minipage}
	\hspace{0.5cm}
	\begin{minipage}[t]{.45\textwidth}	
		\centering
		\subcaption{Generator}
		\begin{tabular}{l} 	
			\text{Configuration per Layer} \\ \hline
			\text{SN Linear $4\times 4\times 128$ output} \\\hline
			\text{ResBlock, up, $\mathtt{ch=128}$}   \\\hline
			\text{ResBlock, up, $\mathtt{ch=128}$}   \\\hline
			\text{Non-local Block ($16\times 16$)}  \\\hline
			\text{ResBlock, up, $\mathtt{ch=128}$}   \\\hline
			\text{BN, RELU}\\\hline
			\text{$3\times 3$ stride $1$ SN Conv. $\mathtt{ch=3}$, Tanh}   \\
			\bottomrule
		\end{tabular}
	\end{minipage}	
\end{table*}

\begin{table*}[t!]
	\centering
	\caption{SA GAN for CELEBA-HQ.\label{tab_sagancelebahq}}\hspace{-0.5cm}
	\begin{minipage}[t]{.45\textwidth}
		\centering	
		\subcaption{Discriminator}
		\begin{tabular}{ll} 	
			\#	&	\text{Configuration per Layer} \\ \hline
			0	&	\text{$3\times3$ stride $1$ SN Conv, $\mathtt{ch=64}$ } \\ \hline
			1	&	\text{ResBlock, down, $\mathtt{ch=128}$ } \\ \hline
			2	&	\text{ResBlock, down, $\mathtt{ch=128}$ } \\ \hline
			3	&	\text{Non-Local Block ($32\times 32$)} \\ \hline
			4	&	\text{ResBlock, down, $\mathtt{ch=256}$ }  \\\hline
			5	&	\text{ResBlock, down, $\mathtt{ch=256}$ }  \\\hline
			6	&	\text{ResBlock, down, $\mathtt{ch=512}$ } \\ \hline
			8	&	\text{ReLU, Global sum pooling} \\ \hline
			9	&	\text{SN Linear 1 output} \\
			\bottomrule
		\end{tabular}
	\end{minipage}
	\hspace{0.5cm}
	\begin{minipage}[t]{.45\textwidth}	
		\centering
		\subcaption{Generator}
	\begin{tabular}{l} 	
		\text{Configuration per Layer} \\ \hline
		\text{SN Linear $4\times 4\times 512$ output} \\\hline
		\text{ResBlock, up, $\mathtt{ch=512}$}   \\\hline
		\text{ResBlock, up, $\mathtt{ch=256}$}   \\\hline
		\text{ResBlock, up, $\mathtt{ch=256}$}   \\\hline
		\text{Non-local Block ($32\times 32$)}  \\\hline
		\text{ResBlock, up, $\mathtt{ch=128}$}   \\\hline				
		\text{ResBlock, up, $\mathtt{ch=64}$}   \\\hline
		\text{BN, RELU}   \\\hline
		\text{$3\times 3$ stride $1$ SN Conv. $\mathtt{ch=3}$, Tanh}   \\
		\bottomrule
	\end{tabular}
	\end{minipage}	
\end{table*}

\begin{table*}[t!]
	\centering
\caption{SA GAN for Tiny-IMAGENET.\label{tab_sagantiny}}\hspace{-0.5cm}
	\begin{minipage}[t]{.45\textwidth}
		\centering	
		\subcaption{Discriminator}
	\begin{tabular}{ll} 	
		\#	&	\text{Configuration per Layer} \\ \hline
		1	&	\text{ResBlock, down, $\mathtt{ch=64}$ } \\ \hline
		2	&	\text{Non-Local Block ($32\times 32$)} \\ \hline
		3	&	\text{ResBlock, down, $\mathtt{ch=128}$ } \\ \hline		
		4	&	\text{ResBlock, down, $\mathtt{ch=256}$ }  \\\hline
		5	&	\text{ResBlock, down, $\mathtt{ch=512}$ }  \\\hline
		6	&	\text{ResBlock, $\mathtt{ch=512}$ } \\ \hline
		8	&	\text{ReLU, Global sum pooling} \\ \hline
		9	&	\text{SN Linear 1 output} \\
		\bottomrule
	\end{tabular}
	\end{minipage}
	\hspace{0.5cm}
	\begin{minipage}[t]{.45\textwidth}	
		\centering
		\subcaption{Generator}
		\begin{tabular}{l} 	
			\text{Configuration per Layer} \\ \hline
			\text{SN Linear $4\times 4\times 512$ output} \\\hline
			\text{ResBlock, up, $\mathtt{ch=512}$}   \\\hline
			\text{ResBlock, up, $\mathtt{ch=256}$}   \\\hline
			\text{ResBlock, up, $\mathtt{ch=128}$}   \\\hline
			\text{Non-local Block ($32\times 32$)}  \\\hline
			\text{ResBlock, up, $\mathtt{ch=64}$}   \\\hline		
			\text{BN, RELU}   \\\hline
			\text{$3\times 3$ stride $1$ SN Conv. $\mathtt{ch=3}$, Tanh}   \\
			\bottomrule
		\end{tabular}
	\end{minipage}	
\end{table*}

\subsection{Network Training Details}\label{subsec:algsettings}
The training details across the datasets (i.e., F-MNIST, CIFAR10, CELEBA-HQ and T-ImageNet) and architectures (i.e., SN DCGAN, and SA GAN) are summarized in Table~\ref{table:hypersetting4all}. %
For both architectures, the decay rate of the (s)BNs at the generator is set to $0.9$. During the evaluation phase, the generator uses the moving averaged mean and variance to produce synthetic samples, thereby being independent of batch size.

\subsection{Other Hyperparameter Settings}\label{subsec:regsettings}

\paragraph{Comparison with SotA on Human Face Synthesis.}
For CELEBA $(64\times64)$, we used the same network architecture as T-ImageNet. This network is not as tailored as PG-GAN~\cite{karras2018progressive} and COCO-GAN~\cite{lin2019cocogan} for human face synthesis. Unlike the other experiments, we followed the FID evaluation of COCO-GAN~\cite{lin2019cocogan} for the sake of fair comparison. The augmentation space is at $\mathtt{feat_{8}}$ (the input of the $4$th residual block). The hyperparameter setting for the $D$ and $G$ optimizers is: $\mathtt{lr_d}=0.0004$, $\mathtt{lr_g}=0.0001$, $\beta_1=0$, $\beta_2=0.999$, $\mathtt{iter_d/iter_g}=1$ and $1$\unit{m} training iterations.
\paragraph{Regularization Techniques in Table~3}
In Sec.~4 of the main paper, we have experimented with a diverse set of regularization techniques and reported the FIDs in Table~3. Their settings are as follows:

For $\mathtt{Label\text{ }smooth.}$, we followed the one-side label smoothing presented in~\cite{SalimansNIPS2016} smoothing the positive labels from $1$ to $0.9$ and leaving the negative ones to $0$ in the binary classification task of the discriminator.

The $\mathtt{GP}$ from~\cite{gulrajani_NIPS2017} and the zero-centered alternative $\mathtt{GP_{zero\text{-}cent}}$ from~\cite{Roth_NIPS2017} are implemented by exploiting the publicly available code in \url{https://github.com/igul222/improved_wgan_training} and \url{https://github.com/rothk/Stabilizing_GANs}. The weighting parameter for $\mathtt{GP}$ and $\mathtt{GP_{zero\text{-}cent}}$ is respectively set to $1$ and $0.1$ as suggested by \cite{Kurach2018GANlandscape,Roth_NIPS2017}. 

When combining $\mathtt{GP}$ with $\mathtt{PA}$, we adjust its weighting factor whenever kicking off a new augmentation level, namely, gradually increasing the weighting factor from zero to its original value within $5$\unit{k} iterations. This is mainly because the new bit can flip the reference label. Such relaxation on the 1-Lipschitz constraint allows the discriminator to timely cope with the new augmentation bit. Using $\beta_2=0.99$ instead of $\beta_2=0.9$ stabilizes the training on SA GAN.

For $\mathtt{Dropout}$, we experimented with different keep rates and applied layers. From Table~\ref{table:dropout}, we selected the best performing setting of the $\mathtt{Dropout}$ with the keep rate $0.7$ applied on the feature space with the spatial dimension $N/4$. 

For $\mathtt{SS}$, we used the same mini-batch construction as in~\cite{ChenSS2019} for computing the auxiliary rotation loss. The rotation loss is respectively added to the $D$ and $G$ loss with the weighting factors equal to $1.0$ and $0.2$ as suggested by~\cite{ChenSS2019}. The augmentation bits does not affect the reference label when constructing the rotation loss.

\paragraph{WGAN-GP}
In Sec.~\ref{subsec:wgan}, we additionally trained CIFAR10 on SN DCGAN with WGAN-GP. The learning rates $\mathtt{lr_d}$ and $\mathtt{lr_g}$ remain the same as that of NS loss, i.e., $2\times 10^{-4}$, but with two discriminator steps per generator step. The two momentum parameters for the Adam optimizer change to $\beta_1 =0$ and $\beta_2=0.9$. The GP is weighted by one.

\begin{table*}[t!]
	\setlength{\tabcolsep}{0.4em} 
	\renewcommand{\arraystretch}{1.1}
	\centering
	\caption{Training details for the experiments in this work.}		\vspace{0.5em}
	\label{table:hypersetting4all}
	\begin{tabular}{c|ccc|ccc} 
		\rowcolor{verylightgray}

		 &  \multicolumn{3}{c|}{\footnotesize{}{\text{$\mathtt{SN\text{ }DCGAN\text{ }NS\text{ }Loss}$}}}&  \multicolumn{3}{c}{\footnotesize{}{\text{$\mathtt{SA\text{ }GAN\text{ }(sBN)\text{ }Hinge\text{ }Loss}$}}}  \tabularnewline 	
		
		\rowcolor{verylightgray}
		
		\multirow{-2}{*}{\footnotesize{}{\text{$\mathtt{Hyper\text{-}parameters}$}}} & \footnotesize{}{$\mathtt{F\text{-}MNIST}$} & \footnotesize{}{$\mathtt{CIFAR10}$} & \footnotesize{}{$\mathtt{CELEBA\text{-}HQ}$}& \footnotesize{}{$\mathtt{CIFAR10}$} & \footnotesize{}{$\mathtt{CELEBA\text{-}HQ}$} & \footnotesize{}{$\mathtt{T\text{-}IMAGENET}$}\tabularnewline 	
		\text{$\beta_1$}  & \text{$0.5$} & \text{$0.5$} & \text{$0.5$}  & \text{$0.0$} & \text{$0.0$} & \text{$0.0$} \tabularnewline 
		\text{$\beta_2$}  &  \text{$0.999$} & \text{$0.999$}  &  \text{$0.999$} & \text{$0.9$}  &  \text{$0.9$} &  \text{$0.9$} \tabularnewline                                                   
		\text{$\mathtt{lr_d}$}  &  \text{$ 10^{-4}$} & \text{$2\times 10^{-4}$} & \text{$2\times 10^{-4}$}  & \text{$3\times10^{-4}$} & \text{$3\times 10^{-4}$}& \text{$3\times 10^{-4}$}\tabularnewline          
		\text{$\mathtt{lr_g}$}  &  \text{$4\times 10^{-4}$} & \text{$2\times 10^{-4}$} & \text{$2\times 10^{-4}$}& \text{$10^{-4}$} & \text{$10^{-4}$}& \text{$10^{-4}$}\tabularnewline    
		\text{$\mathtt{iter_d/iter_g}$} &  \text{$1$} & \text{$1$} & \text{$1$}  & \text{$1$} & \text{$1$}	& \text{$1$}
	\end{tabular}	
	\vspace{-1em}
\end{table*}

\end{document}